\documentclass[11pt]{article}

\makeatletter
\newcommand\RedeclareMathOperator{%
  \@ifstar{\def\rmo@s{m}\rmo@redeclare}{\def\rmo@s{o}\rmo@redeclare}%
}
\newcommand\rmo@redeclare[2]{%
  \begingroup \escapechar\m@ne\xdef\@gtempa{{\string#1}}\endgroup
  \expandafter\@ifundefined\@gtempa
     {\@latex@error{\noexpand#1undefined}\@ehc}%
     \relax
  \expandafter\rmo@declmathop\rmo@s{#1}{#2}}
\newcommand\rmo@declmathop[3]{%
  \DeclareRobustCommand{#2}{\qopname\newmcodes@#1{#3}}%
}
\@onlypreamble\RedeclareMathOperator
\makeatother

\usepackage{witharrows}
\usepackage{tikz}
\usepackage{graphicx,wrapfig,lipsum}   

\usepackage{xspace}

\usepackage{framed}

\usepackage{amsmath,amssymb}
\usepackage{amsthm}
\usepackage{mathtools}
\usepackage{algorithmic}
\usepackage[lined,boxed,ruled,norelsize,algo2e,linesnumbered]{algorithm2e}
\usepackage{url}
\usepackage{fullpage}
\usepackage{makeidx}
\usepackage{enumerate}
\usepackage[top=1in, bottom=1.25in, left=1in, right=1in]{geometry}
\usepackage{graphicx,float,psfrag,epsfig,caption}

\usepackage{hyperref}
\usepackage{epstopdf}
\usepackage{color}
\usepackage{xr}
\usepackage{subcaption}
\usepackage[utf8]{inputenc}

\usepackage{scalerel,stackengine}

\usepackage{thm-restate}

\usepackage{bbm}

\stackMath
\newcommand\reallywidehat[1]{%
\savestack{\tmpbox}{\stretchto{%
  \scaleto{%
    \scalerel*[\widthof{\ensuremath{#1}}]{\kern.1pt\mathchar"0362\kern.1pt}%
    {\rule{0ex}{\textheight}}
  }{\textheight}%
}{2.4ex}}%
\stackon[-6.9pt]{#1}{\tmpbox}%
}
\usepackage[mathscr]{euscript} 

\DeclareSymbolFont{rsfs}{U}{rsfs}{m}{n}
\DeclareSymbolFontAlphabet{\mathscrsfs}{rsfs}
\usepackage{mathrsfs}

\numberwithin{equation}{section}

\newtheoremstyle{myexample} 
    {\topsep}                    
    {\topsep}                    
    {\rm }                   
    {}                           
    {\bf }                   
    {.}                          
    {.5em}                       
    {}  

\newtheoremstyle{myremark} 
    {\topsep}                    
    {\topsep}                    
    {\rm}                        
    {}                           
    {\bf}                        
    {.}                          
    {.5em}                       
    {}  

\newtheorem{claim}{Claim}[section]

\newtheorem{theorem}{Theorem}
\newtheorem{proposition}[claim]{Proposition}
\newtheorem{corollary}[claim]{Corollary}
\newtheorem{definition}[claim]{Definition}
\newtheorem{lemma}[theorem]{Lemma}

\numberwithin{theorem}{section}

\theoremstyle{myremark}
\newtheorem{remark}{Remark}[section]

\theoremstyle{myremark}

\theoremstyle{myexample}

\definecolor{darkgreen}{rgb}{0.0, 0.5, 0.0}

\newcommand{\msedit}[1]{#1}

\hypersetup{
  colorlinks   = true, 
  urlcolor     = blue, 
  linkcolor    = blue, 
  citecolor   = red 
}

\newcommand{\bea}{\begin{eqnarray}}
\newcommand{\eea}{\end{eqnarray}}
\newcommand{\<}{\langle}
\renewcommand{\>}{\rangle}

\newcommand{\wt}{\widetilde}

\def\iid{{\text{i.i.d.~}}}
\def\eg{{\text{e.g.~}}}
\def\ie{{\text{i.e.~}}}

\def\eps{{\varepsilon}}

\def\tmu{\tilde{\mu}}

\def\bmu{{\boldsymbol{\mu}}}

\def\cF{{\mathcal F}}

\def\<{\langle}
\def\>{\rangle}

\def\dv{{\partial v}}

\def\P{\mathbb{P}}

\def\b0{{\boldsymbol{0}}}

\def\Ber{{\text{ Ber}}}

\def\oq{\overline{q}}
\def\uq{\underline{q}}

\def\cA{{\mathcal A}}

\def\Cost{\mathsf{Cost}}

\renewcommand{\b}{\mathbf{b}}

\def\lt{\left}
\def\rt{\right}

\def\eps{\varepsilon}

\def\bbE{{\mathbb{E}}}

\def\bbP{{\mathbb{P}}}
\def\bbR{{\mathbb{R}}}

\def\bbZ{{\mathbb{Z}}}

\def\cA{{\mathcal{A}}}

\def\cF{{\mathcal{F}}}

\RedeclareMathOperator*{\P}{\bbP}

\def\ogamma{\overline{\gamma}}
\def\ugamma{\underline{\gamma}}
\def\of{\overline{f}}
\def\uf{\underline{f}}
\def\oalpha{\overline{\alpha}}

\def\uS{\underline{S}}
\def\eps{\varepsilon}
\def\cA{{\mathcal A}}
\def\Beta{\mathrm{Beta}}
\def\Bin{\mathrm{Bin}}
\def\HG{\mathrm{HyperGeom}}

\newcommand{\alg}{\mathcal{A}}
\newcommand{\adv}{\mathbb{A}}

\def\bbE{{\mathbb{E}}}

\def\bbP{{\mathbb{P}}}
\def\lt{\left}
\def\rt{\right}
\def\tmu{\tilde{\mu}}
\def\bmu{\boldsymbol{\mu}}
\def\strength{\text{strength}}

\title{
Asymptotically Optimal Pure Exploration for \\ Infinite-Armed Bandits
}

\author{
   Xiao-Yue Gong\thanks{
        Operations Research Center, Massachusetts Institute of Technology
    }
    \and
    Mark Sellke\thanks{
        Amazon Core AI
    }
}

\begin{document}

\date{}

\maketitle

\begin{abstract}
We study pure exploration with infinitely many bandit arms generated \iid from an unknown distribution. Our goal is to efficiently select a single high quality arm whose average reward is, with probability $1-\delta$, within $\eps$ of being among the top $\eta$-fraction of arms; this is a natural adaptation of the classical PAC guarantee for infinite action sets. We consider both the fixed confidence and fixed budget settings, aiming respectively for minimal \emph{expected} and \emph{fixed} sample complexity.

For fixed confidence, we give an algorithm with expected sample complexity $O\lt(\frac{\log (1/\eta)\log (1/\delta)}{\eta\eps^2}\rt)$. This is optimal except for the $\log (1/\eta)$ factor, and the $\delta$-dependence closes a quadratic gap in the literature. For fixed budget, we show the asymptotically optimal sample complexity as $\delta\to 0$ is $c^{-1}\log(1/\delta)\big(\log\log(1/\delta)\big)^2$ to leading order. Equivalently, the optimal failure probability given exactly $N$ samples decays as 
$\exp\big(-cN/\log^2 N\big)$, up to a factor $1\pm o_N(1)$ inside the exponent.
The constant $c$ depends explicitly on the problem parameters (including the unknown arm distribution) through a certain Fisher information distance. Even the strictly super-linear dependence on $\log(1/\delta)$ was not known and resolves a question of \cite{grossman2016amplification}.
\end{abstract}

{\small \tableofcontents}

\section{Introduction}

In many learning problems, one faces the classical exploration versus exploitation tradeoff. A central example is the \msedit{(stochastic)} \msedit{multi-armed} bandit \msedit{\cite{Lai-Robbins-85,berry1985bandit}}, where an agent is presented with a set of arms each of which when played gives a stochastic reward from an unknown and arm-dependent distribution. The performance of a bandit algorithm is most commonly determined by its \emph{regret}, \ie the difference between its average reward and the expected reward from the best arm. 
\msedit{Multi-armed bandits and extensions have been applied in many settings including medical trials \cite{berry1995adaptive}, online advertising \cite{li2010contextual}, cognitive radio \cite{anandkumar2011distributed}, and information retrieval \cite{losada2017multi}.}
Optimal algorithms for the \msedit{multi-armed} bandit, including UCB, Thompson sampling, EXP3, and \msedit{various} forms of mirror descent, all make a principled tradeoff between exploration and exploitation.

\msedit{In this work we focus on \emph{pure exploration} bandit problems, a setting motivated by situations where the learning procedure consists of an \emph{initial exploration} phase followed by a \emph{choice} of policy to deploy. This is the case in hyperparameter optimization \cite{li2017hyperband,grossman2016amplification} as well as reinforcement learning from simulated environments. As there is no longer a competing need to exploit, optimal algorithms for pure exploration are different from those minimizing regret.
}

Pure exploration problems were introduced in \cite{even2002pac,mannor2004sample,even2006action} in the probably-approximately-correct (PAC) model. Here given $K$ arms, one adaptively obtains samples until choosing one of the arms to output -- the goal is to ensure that with probability $1-\delta$, this arm has average reward within $\eps$ of the best arm, with minimum possible sample complexity depending on $\eps$ and $\delta$. The early works above focused on the \emph{fixed confidence} setting in which one aims to minimize the expected sample complexity. 
Many subsequent works have also considered the \emph{fixed budget} problem where the sample complexity is uniformly bounded.

While sharp results are known for pure exploration and other bandit problems with $K$ arms, for many applications such as advertising there are far too many arms to explore. This motivated the study of infinite-armed bandit problems in \eg \cite{berry1997bandit,wang2008algorithms}; the pure exploration version was first studied in \cite{aziz2018pure}.
As we review below, several works have studied pure exploration with infinitely many arms, but sharp results were known only in special cases with \eg structural assumptions on the distribution of arms.
The main results of our work give nearly minimax optimal algorithms and lower bounds for infinite-armed pure exploration problems, in both the fixed confidence and fixed budget settings.

\subsection{Problem Formulation}

We now precisely formulate the infinite-armed pure exploration problem of study.
Let $\mathcal{S}=\{a_1,a_2,\dots\}$ be a countably infinite set of stochastic bandit arms indexed by $i=1, 2, \dots$. When arm $a_i$ is sampled, it returns a Bernoulli reward with mean $p_i$.
The values $p_i$ are drawn \iid from \msedit{an arbitrary} reservoir distribution $\mu$ \msedit{supported in $[0,1]$} (which is unknown to the player). We define the cumulative distribution function
\[
    G_{\mu}(\tau) = \bbP^{p\sim \mu}[p\leq \tau]
\]
of $\mu$, and its (left-continuous) inverse
\[
    G^{-1}_{\mu}(p)=\inf\{\tau: G(\tau)\geq p\}.
\]
Finally let $\mu^*=G^{-1}_{\mu}(1)$ denote the essential supremum of $\mu$, \ie the maximum of its support.

An algorithm $\cA$ interacts with $\mathcal S$ in the following way. At each time step $t\in \{1, 2, \dots, T\}$ the algorithm chooses and samples an arm $a_{i_t}\in\mathcal S$, and then observes a Bernoulli reward $r_t\sim \Ber(p_{i_t})$. The reward $r_t$ is independent of previous actions and feedback. Eventually at some time $T$, $\cA$ chooses an arm $a_{i^*}$ to output. If the time-horizon $T=N$ is fixed, we say $\cA$ has a \emph{fixed budget} constraint. If $\bbE[T]\leq N$ is bounded only in expectation, we say $\cA$ has a \emph{fixed confidence} constraint.

For $\eta,\eps,\delta>0$, we say the algorithm $\cA$ is $(\eta,\eps,\delta)$-PAC if 
\begin{equation}
\label{eq:PAC}
    \bbP\big[p_{i^*}
    \geq
    G^{-1}(1-\eta)-\eps\big]\geq 1-\delta
\end{equation}
and set
\[
    \alpha \equiv G^{-1}(1-\eta)
\]
to be the target quantile value.
We emphasize that while $\eta$ is known, $\alpha$ may not be as it depends on the unknown reservoir distribution $\mu$. 
The definition \eqref{eq:PAC} stems from \cite{aziz2018pure}. As brief justification for the parameter $\eta$, note that in an infinite-armed setting the reservoir $\mu$ could give $\eps$-optimal arms with arbitrarily small probability. Thus it is impossible to give a non-asymptotic classical $(\eps,\delta)$-PAC guarantee in our setting without assumptions on $\mu$.
Taking $\eta>0$ as above enables such guarantees by ensuring that a
positive fraction of arms are good enough.

The purpose of this paper is to give $(\eta,\eps,\delta)$-PAC algorithms whose sample complexity $N$ is minimal. We now state our main results, deferring a thorough discussion to Subsection~\ref{subsec:overview}. Let us emphasize that unless explicitly stated, no assumptions on the reservoir distribution $\mu$ are made, nor does the algorithm have any prior knowledge about $\mu$.
Our main result in the fixed confidence case is as follows.

\begin{theorem}
\label{thm:fixedconfidence-intro}
    For any $(\eta,\eps,\delta)$, there exists a $(\eta,\eps,\delta)$-PAC algorithm 
    with expected sample complexity $O\lt(\frac{\log (1/\delta)\log(1/\eta)}{\eta\eps^2}\rt)$.
\end{theorem}

In the fixed budget setting, our interest is especially in the high-confidence regime $\delta\to 0$, where we obtain the following. The following statement is a slightly informal combination of Theorems~\ref{thm:main} and \ref{thm:mainLB}.
We note that while the statement below requires $\alpha$ to be given, this is removable in many cases (e.g. if $\alpha>\frac{1+\eps}{2}$) as discussed extensively in Subsection~\ref{subsec:overview}. 

\begin{theorem}[{Informal}]
\label{thm:fixedbudget-intro}
    For any $(\eta,\eps,\delta)$, let $\alpha\geq \eta$ be given and set $\beta=\alpha-\eps$. Then the optimal $(\eta,\eps,\delta)$-PAC algorithm under fixed budget has sample complexity
    \begin{align}
    \nonumber
    N&= 
    \big(c^{-1}\pm o(1)\big)
    \log(1/\delta)\big(\log\log(1/\delta)\big)^2;
    \\
    \label{eq:calpha}
        c=c_{\alpha,\beta}
        &\equiv
        \frac{\big(\arccos(1-2\alpha)-\arccos(1-2\beta)\big)^2}{2}
        \\
    \label{eq:cintegral}
        &=
        \frac{\left(\int_{\beta}^{\alpha} \frac{dx}{\sqrt{x(1-x)}}\right)^2}{2}
        \,.
    \end{align}
\end{theorem}

An equivalent statement is that at time $N$, the optimal failure probability $\delta$ to have $p_{i^*}\geq G^{-1}(1-\eta)-\eps$ decays as $\exp\lt(-\frac{N(c\pm o(1))}{\log^2(N)}\rt)$. Interestingly the value $\eta$ makes no appearance here, so it is asymptotically irrelevant for the $\delta\to 0$ regime of fixed budget pure exploration.
Interestingly the formula \eqref{eq:cintegral} expresses $c_{\alpha,\beta}$ as the Fisher-information distance between $\alpha$ and $\beta$ in the exponential family of Bernoulli random variables.

\begin{remark}
In our problem formulation above we assumed rewards are Bernoulli, \ie lie in $\{0,1\}$.
More generally one could let each arm have a reward distribution $\nu_i$ over $[0,1]$. One must then consider generalized reservoirs $\mu$, namely probability distributions over such $\nu$.
However as long as the quality of an arm is measured by its mean reward, restricting to Bernoulli rewards loses no generality at all and is just a technical convenience.
This is because any arm with $[0,1]$-valued rewards can be transformed into a Bernoulli arm with $\{0,1\}$-valued rewards and the same mean: simply interpret reward $r\in [0,1]$ as reward $1$ with probability $r$, and as $0$ with probability $1-r$.
See \cite[Section 1.2]{agrawal2012analysis} for further explanation of this point.
\end{remark}

\subsection{Further Notation}

We use the convention that algorithms collect $1$ sample per unit time until terminating, so the time $t$ equivalently denotes the number of total samples collected so far. Denote by $n_{i,t}$ the number of samples of arm $a_i$ collected by time $t$. The $n$-th time $a_i$ is sampled, its reward is $r_{i,n}\in \{0,1\}$. The total reward of arm $i$ up to time $t$ is
\[
    R_{i,t} = \sum_{n=1}^{n_{i,t}} r_{i,n}.
\]
The corresponding average reward is $\hat p_{i,t}=\hat p_i(n_{i,t})=\frac{R_{i,t}}{n_{i,t}}$. We use $o_n(1)$ and $o_N(1)$ to denote quantities tending to $0$ as $n\to\infty$ or $N\to\infty$, with other parameters implicitly held constant. However in Section~\ref{sec:budget} we use \eg $\Omega_{\alpha,\varrho}$ to indicate an asymptotic lower bound with implicit constant factor depending on the values of $\alpha,\varrho$, which are treated as fixed. In all our uses of these notations it is $n$ or $N$ which is tending to infinity while other parameters are always treated as fixed.

\subsection{Related Work}

As discussed above, this work belongs to the area of \emph{pure exploration} for multi-armed bandit problems. Unlike ordinary bandit problems where one aims to minimize the regret compared to the best arm \cite{bubeck2012regret,slivkins2019introduction}, in pure exploration all that matters is the final arm selected by the algorithm. We survey several existing results below, with an emphasis on the high-probability regime of small $\delta$. See \eg Chapter 33 of \cite{lattimore2020bandit} for a more detailed survey.

Pure exploration was first studied in \cite{even2002pac,mannor2004sample,even2006action} in the probably-approximately-correct model. Here given $K$ arms, one adaptively obtains samples until choosing one of the arms to output -- the goal is to ensure that with probability $1-\delta$, this arm has average reward within $\eps$ of the best arm. These works showed that the optimal fixed confidence sample complexity is $\Theta\left(\frac{K}{\eps^{2}} \log \frac{1}{\delta}\right)$.

Later, \cite{bubeck2009pure} considered the \emph{simple regret} of pure exploration problems, namely the regret incurred at the final timestep.
\cite{audibert2010best} studied the closely related problem of identifying the best arm, obtaining nearly tight sample complexity bounds in terms of the the sum of the squared inverse suboptimality-gaps $H=\sum_{i\neq i^*}\Delta_i^{-2}$. 
Further upper and lower sample complexity bounds have been obtained in several works. For example \cite{chen2015optimal,kaufmann2016complexity} show that for fixed confidence, the sample complexity scales as $\Theta\lt(H\log(1/\delta)\rt)$ as $\delta\to 0$. The fixed budget setting, in which the number of adaptive samples is upper-bounded \emph{almost surely} rather than in expectation, turns out to be more difficult. \cite{carpentier2016tight} proved that the optimal fixed budget sample complexity can be $\Theta\lt(H \log(K)\log(1/\delta)\rt)$ as $\delta\to 0$, \ie the fixed budget constraint may lead to an additional $\log(K)$ factor. However it reverts to $\Theta\lt(H\log(1/\delta)\rt)$ when the value of $H$ is known beforehand. Many recent works have studied other aspects of pure exploration, for example by incorporating structured feedback; see \cite{jamieson2014lil,chen2017nearly,katz2020true,kocak2021epsilon,thaker2021pure,alieva2021robust,zhu2022near,aziz2022identifying}.

Infinite-armed bandits have also much received previous study, \eg \cite{berry1997bandit,wang2008algorithms}. Since near-optimal arms may be arbitrarily rare, it is natural to instead compare with a \textbf{quantile} of the arm distribution. For example \cite{chaudhuri2018quantile} aims to minimize regret relative to such a quantile.

The $(\eta,\eps,\delta)$-PAC guarantees we address in this paper were first studied in \cite{aziz2018pure}, for infinitely many arms in the fixed confidence setting. Their approach was to sample $K\asymp \frac{\log(1/\delta)}{\eta}$ arms and then apply a PAC algorithm for $K$-armed pure exploration. As discussed at the beginning of Section~\ref{sec:confidence}, the resulting algorithm ``pays twice'' for the high confidence level $1-\delta$ which leads to a suboptimal $O(\log^2(1/\delta))$ sample complexity upper bound. Top-$k$ extensions were also studied in \cite{ren2019exploring,chaudhuri2019pac}; the $\log^2(1/\delta)$ scaling is still present in their results.

Of particular note is the work \cite{de2021bandits}
which considers also both fixed budget and confidence settings and obtains somewhat similar looking results. However they restrict attention to a special class of reservoir distributions with supremum achieved by an atom of weight $p^*$, which must be $\Delta$-larger than the rest of the support.
This structural assumption of a $\Delta$-gap intrinsically reduces fixed budget sample complexity: their result (see Theorem 4 therein) is actually better than the lower bound we show in Theorem~\ref{thm:mainLB} as there is no $\log\log(1/\delta)$ term (\ie $\log(T)$ in their notation).

From their fixed budget estimate, \cite{de2021bandits} deduce (at the end of Section 1 therein) the same bound as Theorem~\ref{thm:fixedconfidence-intro} in their setting for the special case $(\eta,\eps)=(p^*,\Delta)$. 
Our result hence recovers their while allowing the PAC parameters to vary independently of the reservoir.
Moreover our Theorem~\ref{thm:mainLB} shows that for the general reservoirs we consider, passing from fixed budget to fixed confidence as they do is inherently suboptimal: the factors of $\log\log(1/\delta)$ would remain, but are extraneous for fixed confidence. 

Finally \cite{grossman2016amplification} studied the infinite-arm pure exploration problem where $\alpha$ is given, also focusing on the $\delta\to 0$ asymptotics. They proposed an algorithm with fixed budget sample complexity $\mathcal O\big(\log(1/\delta)\big(\log\log(1/\delta)\big)^2\big)$, and asked whether the $\log\log(1/\delta)$ factors are necessary. Theorem~\ref{thm:mainLB} shows their bound is optimal up to constant factors in terms of $\delta$ and in fact obtains the tight constant. \cite{grossman2016amplification} were motivated by complexity theoretic applications to \emph{amplification} and \emph{derandomization}, where bandit arms correspond to random seeds.

We remark that the analysis in \cite{grossman2016amplification} seems to be technically incomplete. In particular in Lemma 4.5 of (the journal version of) their paper, they neglect to take a union bound over sequences $(T_1,\dots,T_k)$ summing to $T$ but only estimate the probability of each fixed sequence $(T_1,\dots,T_k)$. This is a genuine gap since the number of such sequences is exponentially large in $T$. However their idea to use a moving sequence of rejection thresholds was fundamentally correct and is similar to the main phase of our Algorithm~\ref{alg:main}. We give a fully rigorous, supermartingale-based analysis for our algorithm.

\subsection{Detailed Overview of Main Results}
\label{subsec:overview}

\subsubsection*{Near-Optimal Algorithm with Fixed Confidence}

Our first main result is Theorem~\ref{thm:fixedconfidence-intro} for the fixed confidence problem, where we aim to minimize \textbf{expected} sample complexity. Our algorithm begins by estimating $G^{-1}(1-\eta)$. This is achieved in the following proposition, which is a less explicit statement of Proposition~\ref{prop:alpha}.

\begin{proposition}
\label{prop:alpha-intro}
Given $0\leq \eta_1,\eta_2,\eps,\delta\leq 1/2$ with $\eta_2\leq \eta_1$, there exists an algorithm using $O\lt(\frac{\eta_1\log(1/\eta_2)\log(1/\delta)}{\eta_2^2\eps^2}\rt)$ samples outputting $\hat\alpha\in [0,1]$ such that with probability at least $1-\frac{\delta}{2}$, 
    \[
        \hat \alpha \in \lt[G^{-1}(1-\eta_1)-\frac{\eps}{3},~G^{-1}\lt(1-\eta_1+\eta_2\rt)+\frac{\eps}{3}\rt].
    \]
\end{proposition}

Note that there is a small error tolerance of $\eta_2$ in the quantile value in the result above. This is unavoidable since $\mu$ might put an arbitrarily small amount of mass near $G^{-1}(1-\eta_1)$ in which case the sample complexity to estimate $G^{-1}(1-\eta_1)$ could be arbitrarily large. It is possible to make additional assumptions on the reservoir distribution which rule out such behavior; see e.g. \cite[Assumption $2$]{wang2022beyond} where pure exploration for quantile estimation is studied.

When solving the fixed confidence problem we take $(\eta_1,\eta_2)=(\eta,\eta/2)$. By allowing for error in $\eta$, Proposition~\ref{prop:alpha-intro} works for all probability measures $\mu$ on $[0,1]$. Importantly, the estimate $\hat\alpha$ guaranteed in Proposition~\ref{prop:alpha-intro} turns out to be a good enough proxy for $G^{-1}(1-\eta)$ in the second stage of our algorithm. In this second stage we repeatedly choose a new arm, sample it $\wt O(\eps^{-2})$ times, and accept if the average reward is well above $\hat\alpha$. This leads to the result of Theorem~\ref{thm:fixedconfidence-intro} bounding the fixed confidence sample complexity by $O\lt(\frac{\log (1/\delta)\log(1/\eta)}{\eta\eps^2}\rt)$.

To contextualize this result, we recall the important work \cite{mannor2004sample} which showed $\Theta\lt(\frac{K\log(1/\delta)}{\eps^2}\rt)$ samples are necessary and sufficient for $(\eps,\delta)$-PAC pure exploration in the $K$-armed bandit problem with fixed confidence. Intuitively, one expects these problems to be related via $\eta\approx 1/K$. In fact the following infinite-arm analog was later shown.\footnote{\cite{aziz2018pure} states the result more generally. Specializing to $\bbP[p_i=\alpha]=\eta$ and $\bbP[p_i=\alpha-\eps]=1-\eta$ for any $1/4\leq\alpha\leq 3/4$ recovers the concrete statement of Proposition~\ref{prop:fixedconfidenceLB}.}
It follows that the guarantee of Theorem~\ref{thm:fixedconfidence-intro} is optimal up to the $\log(1/\eta)$ factor.
\begin{proposition}
{\cite[Theorem 1 and Remark 2]{aziz2018pure}}
\label{prop:fixedconfidenceLB}
    There exists an absolute constant $c>0$ such that the following holds. For any $1/4\leq \alpha\leq 3/4$ and $\eta,\delta\leq 1/10$ and for any pure exploration algorithm $\cA$ with expected sample complexity $N\leq \frac{c\log(1/\delta)}{\eta\eps^2}$, there exists a reservoir distribution such that $\cA$ fails to be $(\eta,\eps,\delta)$-PAC.
\end{proposition}

Because the lower bound above is at least $1/\eta$, Theorem~\ref{thm:fixedconfidence-intro} is nearly optimal for any choice of parameters $(\eta,\eps,\delta)$. Namely, the expected sample complexity of our algorithm is at most $O(L\log L)$ where $L=\frac{\log(1/\delta)}{\eta\eps^2}$ is the lower bound in Proposition~\ref{prop:fixedconfidenceLB}. Prior to our work there was a quadratic gap as the best upper bound \cite[Theorem 6]{aziz2018pure} was proportional to $\log^2(1/\delta)$.

\subsubsection*{Optimal Algorithm with Fixed Budget for Known $\alpha$}

Our second main result addresses the fixed budget problem. Throughout, we say that $\cA$ is an $N$-sample algorithm if it acquires exactly $N$ samples almost surely, thereby adhering to the fixed budget constraint. As we show, in this case the sample complexity is strictly superlinear in $\log(1/\delta)$, even if the target value $\alpha$ is known beforehand, and moreover the results hold uniformly for sequences $\alpha_N$ in any compact subset of $(\eps,1)$. This implies that the quantile $\eta$ does not affect the asymptotic sample complexity, since the dependence of Proposition~\ref{prop:alpha-intro} on $\eta$ is dwarfed by the main term.

We first state our results in the easier case that $\alpha$ is given. Theorem~\ref{thm:main} and \ref{thm:mainLB} imply that given a target value $\beta=\alpha-\eps$, and the knowledge that $\bbP[p_i\geq \alpha]$ is bounded below, the best possible asymptotic success probability is 
\[
    1-\exp\left(-\frac{c_{\alpha,\beta}N\cdot(1+o(1))}{\log^2(N)}\right).
\]
We have chosen to write $\beta$ for $\alpha-\eps$ when $\alpha$ is given, since then the value $\eps$ plays no role. Equivalently, the optimal fixed budget sample complexity for an $(\eta,\eps,\delta)$-PAC guarantee is 
\[
    (c_{\alpha,\beta}^{-1}\pm o(1))\log(1/\delta)\log\log(1/\delta)^2.
\]
This recovers the statement in the abstract. The appearance of the factor $\log^2 N$ comes as a surprise and seems to be a new behavior for pure exploration bounds.

\begin{theorem}
\label{thm:main}
For any fixed $0<\beta<\alpha<1$, there is a sequence $(\alg_N)_{N\ge 1}$ of $N$-sample algorithms given explicitly by Algorithm~\ref{alg:main} such that for any $\eta\in (0,1)$ and any sequence of reservoir distributions $\mu_N$ with $G_{\mu_N}^{-1}(1-\eta)\geq \alpha$, 
\begin{align}
    \label{eq:main-algorithm-guarantee}
        \limsup_{N\to\infty} &\frac{(-\log \bbP[p_{i^*}< \beta])\cdot \log^2 N}{N}
        \geq
        c_{\alpha,\beta}\,;
        \\
    \label{eq:calpha}
        c_{\alpha,\beta}
        &\equiv
        \frac{\left(\int_{\beta}^{\alpha} \frac{dx}{\sqrt{x(1-x)}}\right)^2}{2}
        =
        \frac{\big(\arccos(1-2\alpha)-\arccos(1-2\beta)\big)^2}{2}
        \,.
\end{align}
\end{theorem}

\begin{remark}
\label{rem:main-unif}
In fact uniformity in $(\alpha,\beta)$ holds in the following sense. For any sequence $(\alpha_N,\beta_N)_{N\geq 1}$ of pairs with $\min\big(\beta_N,\alpha_N-\beta_N,1-\alpha_N\big)$ uniformly bounded below, there is a sequence $(\alg_N)_{N\ge 1}$ of $N$-sample algorithms such that for any $\eta\in (0,1)$ and any sequence of reservoir distributions $\mu_N$ with $G_{\mu_N}^{-1}(1-\eta)\geq \alpha_N$, 
    \begin{equation}
    \label{eq:main-algorithm-guarantee-uniform}
        \limsup_{N\to\infty} 
        \frac{(-\log \bbP[p_{i^*}< \beta_N])\cdot \log^2 N}
        { c_{\alpha_N,\beta_N}\, N}
        \geq
        1.
    \end{equation}
This can be shown identically to Theorem~\ref{thm:main}, though we don't give the proof in this generality. It is useful for the reduction arguments in Theorems~\ref{thm:alg-1}, \ref{thm:alg-2}, and \ref{thm:alg-3}.
\end{remark}

Conversely, the following lower bound applies for any quantile $\eta\in (0,1)$, and holds even when $\alpha$ is known. It implies that $\eta$ is asymptotically irrelevant for fixed budget sample complexity, i.e. the sample complexity of approximating the $\eta=0.01$-quantile and $\eta=0.99$ quantile in fixed budget pure exploration depends only on the quantile values themselves as $\delta\to 0$.

\begin{theorem}
\label{thm:mainLB}
For any $0<\lambda,\eta<1$ and $0<\beta<\alpha<1$ there exists a \textbf{fixed} reservoir distribution $\mu$ satisfying $\alpha=G_{\mu}^{-1}(1-\eta)$ such that for any sequence of $N$-sample algorithms $\cA_N$,
    \begin{equation}
    \label{eq:main-LB-guarantee}
        \liminf_{N\to\infty} \frac{(-\log\bbP[p_{i^*}< \beta])\cdot \log^2 N}{N}
        \leq 
        c_{\alpha,\beta}+\lambda
        .
    \end{equation}
\end{theorem}

\subsubsection*{Fixed Budget with Unknown $\alpha$}

We consider Theorem~\ref{thm:mainLB} to be a definitive lower bound, since e.g. being given the value of $\alpha$ only makes the result stronger. When $\alpha$ is unknown, it is possible to give an essentially matching algorithm, but more care is required when stating the result. This is inherent and stems from the fact that the value $\alpha=G_{\mu}^{-1}(1-\eta)$ can be difficult or even impossible to estimate, yet determines the constant $c_{\alpha,\beta}$ in the desired rate.

Let us illustrate the issue by a counterexample. Consider $\mu_N$ defined by:
\begin{equation}
\label{eq:psimmuN}
\begin{aligned}
    \bbP^{p\sim \mu_N}[p=0.4] &= \frac{1}{2}+e^{-10N},
    \\
    \bbP^{p\sim \mu_N}[p=0.2] &= \frac{1}{2}-e^{-10N}.
\end{aligned}
\end{equation}
Similarly define $\tmu_N$ by:
\begin{equation}
\label{eq:psimtmuN}
\begin{aligned}
    \bbP^{p\sim \tmu_N}[p=0.4] &= \frac{1}{2}-e^{-10N},
    \\
    \bbP^{p\sim \tmu_N}[p=0.3] &= 2e^{-10N},
    \\
    \bbP^{p\sim \tmu_N}[p=0.2] &= \frac{1}{2}-e^{-10N}.
\end{aligned}
\end{equation}
Then $\mu_N$ and $\tmu_N$ are not distinguishable using $N$ samples, yet $G_{\mu}^{-1}(1/2)=0.4$ while $G_{\tmu}^{-1}(1/2)=0.3$. Using non-distinguishability it follows that the lower bound of Theorem~\ref{thm:mainLB} applies to $\tmu_N$ with threshold $\alpha=G_{\mu_N}^{-1}(1/2)=0.4$, as opposed to the direct application using $G_{\tmu_N}^{-1}(1/2)=0.3$. It is not hard to show using monotonicity of $\frac{1}{\sqrt{x(1-x)}}$ that 
\[
    c_{0.4,0.4-\eps} < c_{0.3,0.3-\eps}
\]
for all $\eps\leq 0.3$. As a result, it is information-theoretically \textbf{impossible} to achieve the rate \eqref{eq:main-algorithm-guarantee} for $\tmu_N$ if the target quantile value $\alpha$ is not given. The core reason is that the value $G_{\tmu}^{-1}(1/2)=0.3$ is too sensitive to the choice $\eta=1/2$ of quantile.

Fortunately, this issue is more of an annoyance than a real difficulty. It can be fixed in several ways. In Theorems~\ref{thm:alg-1}, \ref{thm:alg-2}, and \ref{thm:alg-3} below we give three concrete formulations under which the guarantee \eqref{eq:main-algorithm-guarantee} can be achieved. Informal descriptions are as follows:
\begin{enumerate}
    \item $\alpha$ is taken to be the average of $G_{\mu}^{-1}(\eta)$ for $\eta$ ranging over an interval.
    \item $\mu_N$ is required to satisfy $G_{\mu_N}^{-1}(1-\eta)\geq \frac{1+\eps}{2}$.
    \item $\mu_N=\mu$ is independent of $N$, and the targeted value is
    \[
        \beta=\alpha-\eps=\mu^*-\eps_1
    \]
    for fixed $\eps_1>\eps$. Here $\mu^*$ is the essential supremum of $\mu$, i.e. the maximum value in its support. 
\end{enumerate}

We emphasize that the rate \eqref{eq:main-algorithm-guarantee} is optimal in all cases since the lower bound of Theorem~\ref{thm:mainLB} is for an easier problem. The first formulation above may be the most principled choice. The idea is that an averaged quantile depends continuously on $\mu$, and can in fact be estimated by applying Proposition~\ref{prop:alpha} for several pairs $(\eta_1,\eta_2)$ and computing a Riemann sum. The second formulation requires only the mild condition that $\alpha\geq \frac{1+\eps}{2}$ and uses monotonicity of $c_{\alpha,\alpha-\eps}$ on this set. (In other words, if the average reward values $p$ appearing in \eqref{eq:psimmuN}, \eqref{eq:psimtmuN} were larger than $0.5$, there would be no counterexample.)
The third formulation allows us to almost send $\eta$ all the way down to $0$. It uses the fact that 
\[
    \mu^*-(\eps_1-\eps)\leq G_{\mu}^{-1}(1-\eta')
\] 
for some $\eta'=\eta'(\mu,\eps_1,\eps)>0$. These results show that \eqref{eq:main-algorithm-guarantee} is achievable even without knowledge of $\alpha$, up to a choice of technical modification to sidestep the counterexample discussed above.

\begin{theorem}
\label{thm:alg-1}
    For fixed $\eta_1,\eta_2,\eps$, there is a sequence $(\alg_N)_{N\geq 1}$ of $N$-sample algorithms outputting $a_{i^*}$ such that the following holds for any sequence $(\mu_N)_{N\geq 1}$ of reservoir distributions. Letting
    \[
        \alpha_N=\frac{1}{\eta_1-\eta_2}\cdot \int_{1-\eta_1}^{1-\eta_2}
        G_{\mu_N}^{-1}(x) dx
    \]
    be a quantile average of $\mu_N$, we have
    \begin{equation}
    \label{eq:main-algorithm-guarantee-v1}
        \limsup_{N\to\infty} \frac{(-\log \bbP[p_{i^*}< \alpha_N-\eps])\cdot \log^2 N}{c_{\alpha_N,\alpha_N-\eps}\,N}
        \geq
        1
        .
    \end{equation}
\end{theorem}

\begin{theorem}
\label{thm:alg-2}
    For fixed $\eta,\eps$, there is a sequence $(\alg_N)_{N\geq 1}$ of $N$-sample algorithms outputting $a_{i^*}$ such that for any sequence of reservoir distributions $\mu_N$ satisfying 
    \[
        \alpha_N\equiv G_{\mu_N}^{-1}(1-\eta)\geq \frac{1+\eps}{2},
    \]
    we have
    \begin{equation}
    \label{eq:main-algorithm-guarantee-v2}
        \limsup_{N\to\infty} \frac{(-\log \bbP[p_{i^*}< G_{\mu_N}^{-1}(1-\eta)-\eps])\cdot \log^2 N}{c_{\alpha_N,\alpha_N-\eps}\,N}
        \geq
        1
        .
    \end{equation}
\end{theorem}

\begin{theorem}
\label{thm:alg-3}
    For any fixed $\eps_1>\eps$, there is a sequence $(\alg_N)_{N\geq 1}$ of $N$-sample algorithms outputting $a_{i^*}$ such that for any fixed reservoir distribution $\mu$ with $\mu^*>\eps$, 
    \begin{equation}
    \label{eq:main-algorithm-guarantee-v3}
        \limsup_{N\to\infty}\frac{(-\log \bbP[p_{i^*}< \mu^*-\eps_1)\cdot \log^2 N}{N}
        \geq 
        c_{\mu^*,\mu^*-\eps}
        .
    \end{equation}
\end{theorem}

Finally in Section~\ref{subsec:many-good-fixed-budget} we observe that our fixed budget algorithm can actually output $\log N$ good arms with the same success probability as for outputting a single good arm.

\section{The Fixed Confidence Setting}
\label{sec:confidence}

Recall that we focus on the $(\eta,\eps,\delta)$ notion of optimality, aiming to output an arm which is within $\eps$ of the top $\eta$-quantile with probability $1-\delta$. The first challenge in this problem is that the arm distribution and in particular the value of the desired $\eta$-quantile is unknown. The first phase of our fixed confidence algorithm aims to estimate this quantile value. The second phase then aims to find a single arm which is almost as good as this estimate with high probability.

Focusing on the $\delta$-dependence, a challenge with infinitely many arms is that to succeed with probability $1-\delta$, it is necessary both to sample $\log(1/\delta)$ arms to ensure a good arm is ever observed, and to sample an arm $\log(1/\delta)$ times to make it safe to output. The approach of \cite{aziz2018pure} thus requires $O(\log^2(1/\delta))$ samples because it obtains $O(\log(1/\delta))$ samples each of $O(\log(1/\delta))$ arms. However in our algorithm, the first phase samples $O(\log(1/\delta))$ arms $O(1)$ times each, while the second phase samples $O(1)$ arms $O(\log(1/\delta))$ times each. This allows us to satisfy both necessary conditions above without paying twice for the $(1-\delta)$ level of confidence.

\subsection{Estimating $\alpha$}

\begin{figure}[!h]
\SetKwFor{Loop}{loop}{}{end}
\begin{algorithm2e}[H]
\label{alg:alpha}
\caption{Output $\hat \alpha \in \big[G^{-1}(1-\eta_1)-\frac{\eps}{3},G^{-1}\lt(1-\eta_1+\eta_2\rt)+\frac{\eps}{3}\big]$ with probability $1-\frac{\delta}{2}$ }

\SetAlgoLined\DontPrintSemicolon

\textbf{input: }arm set $\mathcal S=(a_1,a_2,\dots)$ and parameters $(\eta_1,\eta_2,\eps,\delta)\in (0,1)$ with $\eta_2<\eta_1$.
\\
\text{initialize: }$K=\frac{C\eta_1\log (1/\delta)}{\eta_2^2}$.
\\
\For{$i=1,2,\dots,K$}{
    Collect $n=\frac{C\log(1/\eta_2)}{\eps^2}$ samples of arm $i$.
    Set $\hat p_i=\hat p_i(n)$ the empirical average reward.
}
Let $\hat\alpha$ be the $k$-th largest value in $\{\hat p_1,\dots,\hat p_K\}$ for $k= \lceil K\big(\eta_1-\frac{\eta_2}{2}\big)\rceil$.
\\
Return $\hat\alpha$
\end{algorithm2e}
\end{figure}

We first give in Alg.~\ref{alg:alpha} a simple procedure to estimate the top $\eta_1$ quantile, allowing an $\eps/3$ error as well as an $\eta_2\leq \eta_1$ error in the quantile itself. Alg.~\ref{alg:alpha} obtains $O\lt(\frac{\log(1/\eta_2)}{\eps^2}\rt)$ samples from each of the first $K=O\lt(\frac{\eta_1\log (1/\delta)}{\eta_2^2}\rt)$ arms $a_1,\dots,a_K$. The resulting estimator $\hat\alpha$ is the $1-\eta_1+\frac{\eta_2}{2}$ quantile of the empirical average rewards $\hat p_1,\dots,\hat p_k$. Its main guarantee is below. We note that taking $\eta_2=\eta_1/2$ suffices for fixed confidence, but the greater generality is helpful later in Subsection~\ref{subsec:reduction}.

\begin{proposition}
\label{prop:alpha}
Fix $0\leq \eta_1,\eta_2,\eps,\delta\leq 1$ with $\eta_2\leq \eta_1$. With probability at least $1-\frac{\delta}{2}$, the output $\hat\alpha$ of Alg.~\ref{alg:alpha} satisfies
\[
        \hat \alpha \in \lt[G^{-1}(1-\eta_1)-\frac{\eps}{3},G^{-1}\lt(1-\eta_1+\eta_2\rt)+\frac{\eps}{3}\rt].
    \]
    Moreover, Alg.~\ref{alg:alpha} has sample complexity 
    \[
    O\lt(\frac{\eta_1\log(1/\eta_2)\log(1/\delta)}{\eta_2^2\eps^2}\rt).
    \]
\end{proposition}

\begin{proof}
The sample complexity is clear so we focus on the first statement. First observe that by a Chernoff estimate, for each $i\in [K]$, 
\begin{equation}
\label{eq:pi-near-hatpi}
    \bbP\lt[|p_i-\hat p_i|\geq \frac{\eps}{3}\rt]\leq \frac{\eta_2}{8}.
\end{equation}
Let $N(\eps)$ be the number of $i\in [K]$ such that $|p_i-\hat p_i|\geq \frac{\eps}{3}$. Applying a second Chernoff estimate (of multiplicative form, see e.g. \cite[Theorem 4.5]{mitzenmacher2017probability}) on these events as $i$ varies and noting that $K\eta_2\geq C\log(1/\delta)$, \eqref{eq:pi-near-hatpi} implies
\begin{equation}
\label{eq:N-eps-bound}
    \bbP\lt[N(\eps)\leq \frac{K\eta_2}{6}\rt]\geq 1-\frac{\delta}{8}.
\end{equation}
We next show that with probability at least $1-\frac{\delta}{4}$, 
\begin{equation}
\label{eq:oalpha-bound}
    \hat\alpha \leq \oalpha+\frac{\eps}{3}\equiv G^{-1}\lt(1-\eta_1+\eta_2\rt)+\frac{\eps}{3}.
\end{equation}
With $p_i$ the (true) mean reward from arm $a_i$, let
\[
    N_{\oalpha}\equiv
    \lt|
    \{i\in [K]~:~p_i > \oalpha\}
    \rt|
\] 
denote the number of the $K$ tested arms which satisfy $p_i> \oalpha$. By definition, $N_{\oalpha}$ is stochastically dominated by a $\Bin\lt(K,\eta_1-\frac{9\eta_2}{10}\rt)$ random variable, and $\eta_1-\frac{3\eta_2}{4}=\Theta(\eta_1)$ since $\eta_2\leq\eta_1$. Note that
\begin{align*}
    \eta_1-\frac{9\eta_2}{10}&\asymp \eta_1-\frac{3\eta_2}{4}\asymp \eta_1,
    \\
    \frac{\eta_1-\frac{9\eta_2}{10}}{\eta_1-\frac{3\eta_2}{4}}
    &\geq 1+\frac{\eta_2}{20\eta_1}.
\end{align*}
Therefore another multiplicative Chernoff estimate implies 
\[
    \bbP\lt[N_{\oalpha}\leq K\lt(\eta_1-\frac{3\eta_2}{4}\rt)\rt]
    \geq
    e^{-\Omega(K\eta_2^2/\eta_1)}
    \geq
    1-\frac{\delta}{8}.
\]
When both $N(\eps)\leq \frac{K\eta_2}{6}$ and $N_{\oalpha}\leq K\lt(\eta_1-\frac{3\eta_2}{4}\rt)$ hold, it follows by definition that $\hat\alpha \leq \oalpha+\frac{\eps}{3}$. Hence recalling \eqref{eq:N-eps-bound} above, we conclude that
\[
    \bbP\lt[\hat\alpha \leq \oalpha+\frac{\eps}{3}\rt]\geq 1-\frac{\delta}{4},
\]
establishing \eqref{eq:oalpha-bound}. The other direction is similar. With $\alpha =G^{-1}(1-\eta_1)$ as usual, we set
\begin{equation}
\label{eq:ualpha-bound}
    N_{\alpha}\equiv
    \lt|
    \{i\in [K]:p_i\geq \alpha\}
    \rt|
    .
\end{equation}
This time, $N_{\alpha}$ stochastically dominates a $\Bin(K,\eta_1)$ random variable. Yet another Chernoff estimate yields
\[
    \bbP\lt[N_{\alpha}\geq K\lt(\eta_1-\frac{\eta_2}{4}\rt)\rt]
    \geq
    1-\frac{\delta}{8}.
\]
Using \eqref{eq:N-eps-bound} in the same way as above, we find
\[
    \bbP\lt[\hat\alpha \geq \alpha-\frac{\eps}{3}\rt]\geq 1-\frac{\delta}{4}.
\]
This concludes the proof.
\end{proof}

\subsection{The Algorithm for Fixed Confidence}

\begin{figure}[!h]
\SetKwFor{Loop}{loop}{}{end}
\begin{algorithm2e}[H]
\label{alg:fixedconfidence}
\caption{Output $a_{i^*}$ such that $p_{i^*}\geq \hat\alpha-\eps$ with probability $1-\delta$. }

\SetAlgoLined\DontPrintSemicolon

\textbf{input: }arm set $\mathcal S=(a_1,a_2,\dots)$ and parameters $(\eta,\eps,\delta,\hat \alpha)$
\\
\For{$i=K+1,K+2,\dots, K+\frac{C\log(1/\delta)}{\eta}$}{
    Collect $\frac{C\log (1/\eta\delta)}{\eps^2}$ samples of arm $i$.
    Set $\hat p_i=\hat p_i(n)$ the empirical average reward.
    \\
    \If{$\hat p_i \geq \hat\alpha - \frac{\eps}{3}$}{
    Return $a_i$
    }
}
\end{algorithm2e}
\end{figure}

Alg.~\ref{alg:fixedconfidence} repeatedly chooses a new arm $a_i$ and obtains $O\lt(\frac{\log (1/\eta\delta)}{\eps^2}\rt)$ samples. It accepts if the sample mean was at least $\hat\alpha - \frac{\eps}{3}$, and otherwise moves on to the next arm. If $\frac{C\log(1/\delta)}{\eta}$ arms have been tried without success, then Alg.~\ref{alg:fixedconfidence} outputs no arm, thus declaring failure. This termination condition is necessary to avoid incurring huge sample complexity when Alg.~\ref{alg:alpha}'s estimate $\hat\alpha$ of $\alpha$ is inaccurate. We now prove Theorem~\ref{thm:fixedconfidence-intro}, restated more precisely as follows.

\begin{theorem}
\label{thm:fixedconfidence-body}
    Apply Alg.~\ref{alg:alpha} with parameters $(\eta_1,\eta_2,\eps,\delta)=(\eta,\eta/2,\eps,\delta)$, and then apply
    Alg.~\ref{alg:fixedconfidence} using the resulting value $\hat\alpha$. This combined algorithm has expected sample complexity $O\lt(\frac{\log (1/\eta)\log(1/\delta)}{\eta\eps^2}\rt)$. Moreover its output $a_{i^*}$ satisfies 
    \[
        \bbP[p_{i^*}\geq G^{-1}(1-\eta)-\eps]\geq 1-\delta.
    \]
\end{theorem}

\begin{proof}
    First we analyze the expected sample complexity. On the event that 
    \begin{equation}
    \label{eq:hatalphagood}
        \hat \alpha \in \lt[G^{-1}(1-\eta)-\frac{\eps}{3},G^{-1}\lt(1-\frac{\eta}{2}\rt)+\frac{\eps}{3}\rt]
    \end{equation}
    we claim that Alg.~\ref{alg:fixedconfidence} terminates with probability $\eta/4$ for each $a_i$. Indeed, if 
    \[
        \hat p_i\geq G^{-1}\lt(1-\frac{\eta}{2}\rt)
    \]
    then termination always happens by definition. This has probability at least $1/4$ if $p_i\geq G^{-1}\lt(1-\frac{\eta}{2}\rt)$ by \cite[Theorem 1]{greenberg2014tight}, and the latter condition has probability at least $\eta/2$ by definition. It follows that when \eqref{eq:hatalphagood} holds, the expected sample complexity of Alg.~\ref{alg:fixedconfidence} is $O\lt(\frac{\log (1/\eta\delta)}{\eta\eps^2}\rt)$.  On the other hand, \eqref{eq:hatalphagood} fails to hold with probability less than $\delta$. Because of the explicit termination condition in Alg.~\ref{alg:fixedconfidence}, this yields a additional sample complexity contribution of smaller order $O\lt(\delta \log(1/\delta)\frac{\log (1/\eta\delta)}{\eta\eps^2}\rt)$. Finally Alg.~\ref{alg:alpha} has sample complexity 
    \[
        O\lt(\frac{\log (1/\eta)\log(1/\delta)}{\eta \eps^2}\rt)
    \]
    which clearly forms the dominant contribution.
    This completes the proof of the sample complexity bound and we now turn to proving correctness with probability $1-\delta$. First, it is easy to see that Alg.~\ref{alg:alpha} outputs some arm $a_i$ with probability at least $1-\frac{\delta}{2}$. It therefore suffices to show that for any fixed $\hat\alpha$ satisfying \eqref{eq:hatalphagood}, conditioned on the event $\hat p_i\geq \hat\alpha-\frac{\eps}{3}$, the conditional probability that $p_i\geq \alpha-\eps$ is at least $1-\frac{\delta}{2}$. 

    We do this using Bayes' rule. If $p_i\geq G^{-1}(1-\frac{\eta}{2})$, then as above \cite[Theorem 1]{greenberg2014tight} implies
    \[
        \bbP\lt[\hat p_i\geq \hat\alpha-\frac{\eps}{3}\rt]
        \geq 
        \bbP[\hat p_i \geq p_i]
        \geq
        1/4.
    \]
    This event hence contributes probability at least $\eta/4$ to the event $p_i\geq G^{-1}(1-\eta)$. On the other hand, if $p_i\leq G^{-1}(1-\eta)-\eps \leq \hat \alpha -\frac{2\eps}{3}$, then 
    \[
        \bbP\lt[\hat p_i\geq \hat\alpha-\frac{\eps}{3}\rt]
        \leq 
        \bbP\lt[\hat p_i \geq p_i+\frac{\eps}{3}\rt]
        \leq
        \eta\delta/8
    \]
    for an absolute constant $C$. Combining these via Bayes' rule implies the desired result.
\end{proof}

\begin{remark}
Our fixed confidence algorithm, given by combining Alg.~\ref{alg:alpha} with Alg.~\ref{alg:fixedconfidence} as above, requires only $O(1)$ \emph{batches} on average. Here in each \emph{batch} we choose $b$ arms to sample exactly $s$ times each. Minimizing the number of required batches is often desirable, see e.g. \cite{perchet2016batched,gao2019batched} for results on regret minimization. In particular Alg.~\ref{alg:alpha} uses a single batch with $s_1=\frac{C\log(1/\eta_2)}{\eps^2}$ samples of $b_1=\frac{C\log(1/\delta)}{\eta}$ arms. Then Alg.~\ref{alg:fixedconfidence} can be implemented in a batched way with $s_2=\frac{C\log(1/\eta\delta)}{\eps^2}$ and a sequence of batch sizes $b_{2,i}=\frac{2^i}{\eta}$ for $1\leq i\leq \log_2\lt(\frac{C\log(1/\delta)}{\eta}\rt)$. (In the latter phase, one stops after finding an arm to accept.) It is unclear how limiting the number of batches would require one to modify Algorithm~\ref{alg:main} for the fixed budget problem.
\end{remark}

\section{Lower Bound for Fixed Budget}

Here we prove Theorem~\ref{thm:mainLB}. For any $\alpha,\beta,\eta,\varrho>0$ we construct a reservoir $\mu=\mu_{\alpha,\beta,\eta,\varrho}$ such that
\begin{equation}
\label{eq:LB-varrho}
    \liminf_{N\to\infty} \frac{(-\log \bbP^{\mu}[p_{i^*}< \beta])\cdot \log^2 N}{N}
        \leq 
        c_{\alpha,\beta}+\lambda(\varrho)
\end{equation}
holds for any sequence of $N$-sample algorithms $\cA_N$, and where $\lim_{\varrho\to 0}\lambda(\varrho)=0$ for fixed $\alpha,\beta,\eta$.


\subsection{Admissible Reservoirs and Bayesian Perspective}

In proving Theorem~\ref{thm:mainLB}, we will use reservoir distributions $\mu$ of a specific form. Namely, we require each $\mu$ to be supported on an interval $[\ugamma,\ogamma]$, where
\[
    0<\beta-\varrho<\ugamma<\beta<\alpha<\ogamma<\alpha+\varrho<1.
\]
In fact we define $\ugamma,\ogamma$ explicitly (recall that $\varrho>0$ is a small constant we will send to $0$) by
\begin{equation}
\label{eq:ogamma-def}
\begin{aligned}
    \theta(\ugamma)&=\theta(\beta)-\varrho^2;\\
    \theta(\ogamma)&=\theta(\alpha)+\varrho^2.
\end{aligned}
\end{equation}

We say $\mu$ is $(\ugamma,\ogamma,\uf,\of)$ \emph{admissible} if $\mu$ has density $\mu(dx)=f(x)dx$ for a Borel measurable function $f$ and satisfies for constants $0<\uf<\of<\infty$,  
\[
    f(x)\in [\uf,\of],\quad\forall x\in [\ugamma,\ogamma].
\]
Towards proving Theorem~\ref{thm:mainLB}, we fix throughout this section some $(\ugamma,\ogamma,\uf,\of)$ admissible $\mu$ such that $G^{-1}_{\mu}(\alpha)=\eta$ holds, for appropriate constants $(\uf,\of)$ depending only on $(\eta,\eps,\alpha,\beta,\ugamma,\ogamma)$. It is easy to see that this is always possible.

An admissible $\mu$ is roughly comparable to the uniform distribution on an interval. Using admissible reservoirs gives each $a_i$ the potential to slowly degrade in observed quality over time. We remark that while it is more convenient to work with reservoirs supported away from the boundaries, i.e. in $[\ugamma,\ogamma]\subseteq (0,1)$, we do not expect this to be essential.

It will be helpful throughout this section to take a Bayesian point of view. We treat $\mu_N$ as known to $\cA_N$, since $\cA_N$ is in fact allowed to depend on $\mu_N$. Thus at each time $t$, each $p_i$ has a posterior probability distribution which we denote by $\mu_{i,t}$. Note that each $\mu_{i,t}$ depends only on $(n_{i,t},\hat p_{i,t})$ and is initialized at $\mu_{i,0}=\mu$. We denote by 
\begin{equation}
\label{eq:bmu}
    \bmu^t=(\mu_{1,t},\mu_{2,t},\dots)
\end{equation}
the sequence of posterior distributions $\mu_{i,t}$. Since arms are independent, $\bmu^t$ is the full time-$t$ posterior of the algorithm.

\subsection{Batched Algorithms and Adversaries}

In pure exploration problems, it is possible to significantly simplify the structure of any algorithm at the cost of a small multiplicative increase in the sample complexity. We carry this out using the notion of a batch-compressed algorithm.

\begin{definition}
\label{defn:batch}
Given an increasing sequence $B=(b_1,b_2,\dots)$ of positive integers, an algorithm $\alg$ is \textbf{$B$-batch-compressed} if $\alg$ can only act by increasing the number of times $n_i$ that $a_i$ has been sampled from $b_k$ to $b_{k+1}$, so that $n_i\in B$ holds at all times. $B$ is \textbf{$\varrho$-slowly increasing} if
\[
    \frac{b_{k+1}}{b_k + 1}\leq 1+\varrho,\quad\forall k\geq 1.
\]
Finally if $\cA$ is $B$-batch-compressed and $B$ is $\varrho$-slowly increasing, we say that $\cA$ is $\varrho$-batch-compressed.
\end{definition}

Unlike the batched algorithms studied in  \cite{perchet2016batched,gao2019batched}, batch-compression is only important for us as an analysis technique. Indeed the following proposition shows that it does not fundamentally affect pure exploration algorithms.

\begin{proposition}
\label{prop:slowly-increasing}
If $B$ is $\varrho$-slowly increasing, then for any $N$-sample algorithm $\alg$, there exists an $B$-batch-compressed $\lfloor N(1+\varrho)\rfloor$ algorithm $\alg'$ with the same output.
\end{proposition}

\begin{proof} 
We show how to simulate $\alg$ using the $B$-batch-compressed $\alg'$, assuming that the sequence of rewards for each $a_i$ is fixed. Each time $\alg$ samples arm $i$ for the $n_i=(a_k+1)$-st time for $a_k\in A$, $\alg'$ samples arm $i$ until $n_i=a_{k+1}$. Then $\alg'$ has all the information of $\alg$ at all times, hence can simulate the behavior and output of $\alg$. Moreover by the definition of $\varrho$-slowly increasing, the sample complexity of $\alg'$ is larger than that of $\alg$ by at most a factor $(1+\varrho)$.
\end{proof}

We will use the above with $\varrho\to 0$ slowly as $N\to \infty$. Then the sample complexity increase $1+\varrho$ is absorbed into the $1+o(1)$ factor in Theorem~\ref{thm:mainLB}. As a result it suffices to establish \eqref{eq:LB-varrho} under the additional assumption that $\cA_N$ is $\varrho$-batch-compressed.

An unusual feature of our setting is that we aim to prove a very \emph{small} lower bound on the failure probability of an algorithm. By contrast, in most algorithmic impossibility results, one aims to prove that algorithms fail with high probability. An important contribution of our lower bound is to introduce a framework to prove such lower bounds beyond one-step change of measure arguments. The key is to consider adversaries with bounded ``strength'' to distort the random feedback.

\begin{definition}
An \textbf{adaptive randomness distorting adversary} $\adv$ interacts with a $B$-batch-compressed algorithm $\cA$ in the following way.
Suppose $\cA$ chooses to increase the number of samples of arm $a_i$ from $b_k$ to $b_{k+1}$. Then $\adv$ may restrict the set of possible outcomes of these $b_{k+1}-b_k$ samples. Additionally, when $\cA$ outputs an arm $a_{i^*}$, the adversary can restrict the possible values of $p_{i^*}$.
\end{definition}

We will refer to adversarial actions as \emph{declarations}. Thus when $\cA$ chooses a batch of samples, $\adv$ may declare that some property holds for the observed rewards.

As defined above, an adversary $\adv$ can do anything. We will limit the power of $\adv$ to make \textbf{low-probability} declarations. To formalize this, we charge $\adv$ per ``bit" of probabilistic distortion, and give $\adv$ a deterministic ``budget" for doing so. We enforce this from a Bayesian point of view: the reservoir distribution $\mu_N$ is known to both $\cA$ and $\adv$, but neither has any information on the true reward probabilities $p_i$ beyond the observed rewards. Thus $\cA$ and $\adv$ share at any time $t$ the posterior distribution $\bmu^t$. In particular recalling~\eqref{eq:bmu}, $\bmu^t$ determines the distribution for the outcome of the next batch of $b_{k+1}-b_k$ samples.

\begin{definition}
Suppose that at time $t$, the declaration of $\adv$ has probability $P_t$ to hold according to $\bmu^t$. Let the sum
\begin{equation}
\label{eq:cost}
    \Cost_t=\sum_{s\leq t}\log(1/P_s)
\end{equation}
be the total cost of $\adv$ up to time $t$, and $\Cost_N$ the total cost of $\dv$. We say $\strength(\adv)\leq \Cost$ holds for some $\Cost\in\mathbb R$ if the bound $\Cost_N\leq \Cost$ almost surely. 
\end{definition}

The next key lemma shows that to obtain a lower bound for the failure probability of an algorithm, it suffices to prevent success using a low strength $\adv$.

\begin{lemma}
\label{lem:adversary-abstract}
Suppose there exists a randomness distorting adversary $\adv$ of strength $\Cost$ whose declarations ensure that any exploration algorithm $\alg$ outputs $i^*$ satisfying $p_{i^*}\leq \beta$ almost surely. Then the true failure probability of $\alg$ is
\[
    \bbP^{\mu_N,\alg}[p_{i^*}\leq \beta]\geq 
    e^{-\Cost}.
\]
\end{lemma}

\begin{proof} 
In the Bayesian viewpoint and conditioning on $\adv$'s declarations holding, the sequence
\[
    M_t\equiv
    \mathbb P^{\bmu^t}[p_{i^*}\leq \beta] 
    \cdot \prod_{s\leq t}P_s
\]
forms a $[0,1]$-valued supermartingale, i.e. $\bbE[M_{t+1}~|~\cF_t]\leq M_t$. This is because conditioning on an event with probability $P_t$ can increase the probability of an event by at most a factor $1/P_t$. Comparing $M_0$ and $M_N$ yields the result. 
\end{proof}

Our approach to the lower bound is to construct an adversary who declares that in each batch of samples, the empirical average reward $\hat p_i(n_{i,t})$ of arm $i$ drops by $\Omega\left(\frac{\varrho}{\log(N)}\right)$, at least for say $n_{i,t}\geq N^{\varrho}$. This forces the average reward of any arm to become small once it has been sampled $\Omega(N^{1-\varrho})$ times. Moreover it follows from the moderate deviation rates for the binomial distribution that this adversary pays $O\left(\frac{1}{\log^2(N)}\right)$ cost per sample.

\subsection{Fisher Information Distance}
\label{subsec:fisher}

Determining the tight constant $c_{\alpha,\beta}$ requires significant care. In particular the adversary must decrease the empirical average rewards $\hat p_{i,t}$ at a precise rate depending on $n_{i,t}$. This rate turns out to involve the \emph{Fisher information distance}. For $a,b\in [0,1]$ we define the Fisher information distance $d_F(a,b)$ between $a$ and $b$ to be 
\[
    d_F(a,b)=\left|\int_a^b \frac{dx}{\sqrt{x(1-x)}}\right|
    .
\] 
This agrees with the more general Fisher information metric when each $a\in [0,1]$ is identified with the corresponding Bernoulli distribution.
We refer the reader to \cite{nielsen2020elementary} for a survey on information geometry. In short, the Fisher information yields a natural Riemannian metric on families of probability distributions which are parametrized by smooth manifolds. However we will use only elementary properties of $d_F$.

We parametrize $[0,1]$ using the function $\theta:[0,1]\to [0,\pi]$ defined by 
\begin{equation}
\label{eq:theta}
    \theta(a)=d_F(0,a)
    =\int_0^{a}\frac{dx}{\sqrt{x(1-x)}}
    =\arccos(1-2a).
\end{equation}
In particular, 
\[
    d_F(a,b)=|\arccos(1-2a)-\arccos(1-2b)|\geq
    2|a-b|
\] 
and so $d_F(0,1)=\pi$. The main property of $\theta$ that we will use is the resulting differential equation
\begin{equation}
\label{eq:theta-ODE}
    \theta'(a)=\frac{1}{\sqrt{\theta(a)(1-\theta(a))}}.
\end{equation}
In our case, $\theta^{-1}$ parametrizes a ``constant speed" path through the space of Bernoulli variables, viewing the Fisher information. Correspondingly, our adversary will ensure that $\theta(\hat p_i(n_{i,t}))$ decreases linearly in $\log(n_{i,t})$. 


\subsection{Preliminary Lemmas from Moderate Deviations}

Recall that for positive integers $a$ and $b$, the $\Beta(a,b)$ distribution has probability density function
\[
\frac{(a+b-1)!}{(a-1)! (b-1)!} x^{a-1}(1-x)^{b-1}
\]
for $x\in [0,1]$. We now recall a moderate deviations principle for the binomial distribution and a central limit theorem for the beta distribution.

\begin{lemma}[{\cite[Theorem 2.2]{de1992moderate}}]
\label{lem:moderate}
For any $0<\uq<\oq<1$ and constant $\varrho>0$ there exist $n_0(\uq,\oq,\varrho), \Delta_0(\uq,\oq,\varrho)$ and $M_0(\uq,\oq,\varrho)$ such that the following holds for all $p\in [\uq,\oq]$. For $n\geq n_0$ sufficiently large and any $\frac{1}{\Delta_0 \sqrt{n}}\leq \Delta\leq \Delta_0$ we have
\[
    e^{\left(-\frac{\Delta^2 }{2p(1-p)}-\varrho\right)n}
    \leq
    \mathbb P\left[\frac{Bin(n,p)}{n}\leq p- \Delta\right] \leq e^{\left(-\frac{\Delta^2 }{2p(1-p)}+\varrho\right)n}.  
\]
\end{lemma}

\begin{lemma}[{\cite[Lemma A.1]{moscovich2016exact}}]
\label{lem:beta-clt}
Let $\{a_n\}_{n\geq n_0}$ be a sequence satisfying
\[
    \ugamma\leq \frac{a_n}{n}\leq \ogamma.
\]
Then the $\Beta(n-a_n+1,a_n+1)$ distribution on $[0,1]$ obeys a central limit theorem with mean $\frac{a_n}{n}$ and standard deviation $\sqrt{\frac{(a_n/n)(1-(a_n/n))}{n}}$ in the sense that for any bounded sequence $(w_n)_{n\geq n_0}$ of real numbers and with $\Phi$ the normal CDF,
\[
    \lim_{n\to\infty}\lt|\Phi(w_n)-\mathbb P^{x\sim \Beta(n-a_n+1,a_n+1)}\left[
    \big(x-(a_n/n)\big)\cdot \sqrt{\frac{n}{(a_n/n)(1-(a_n/n))}}
    \leq w_n
    \right]
    \rt|
    =
    0.
\]
\end{lemma}

In the next two lemmas, we lower bound the probability that $\hat p_{i,t}$ changes significantly when the number $n_{i,t}$ of samples for $a_i$ increases by a factor $(1+\varrho)$.

\begin{lemma}
\label{lem:might-be-ok}

Assume $\mu$ is $(\ugamma,\ogamma,\uf,\of)$-admissible. Suppose that arm $i$'s average reward $\hat p_{i,t}$ after $n=n_{i,t}$ samples satisfies
\begin{equation}
\label{eq:pitbetaogamma}
    \hat p_{i,t}\in [\beta,\ogamma].
\end{equation}
Then for $n\geq C(\ugamma,\ogamma,\uf,\of,\beta)$ sufficiently large,
\begin{equation}
\label{eq:good-chance-to-be-ok}
    \mathbb P^{x\sim \mu_{i,n}}\big[x\leq \hat p_{i,t}\big]\geq \frac{\uf}{3\of}.
\end{equation}
\end{lemma}

\begin{proof} 
Let $R_{i,t}=n\hat p_{i,t}$ be the total reward from arm $i$ so far. The posterior distribution $\mu_{i,t}$ for $p_i$ takes the form
\[
    \mu_{i,t}(dx)
    =
    \frac{x^{R_{i,t}} (1-x)^{n-R_{i,t}} f(x)dx}
    {\int_{\ugamma}^{\ogamma}x^{R_{i,t}} (1-x)^{n-R_{i,t}} f(x)dx}.
\]
For $x\in [\ugamma,\ogamma]$ we estimate
\[
    \frac{x^{R_{i,t}} (1-x)^{n-R_{i,t}} f(x)}
    {\int_{\ugamma}^{\ogamma}x^{R_{i,t}} (1-x)^{n-R_{i,t}} f(x)dx}
    \geq
    (\uf/\of)\cdot 
    \frac{x^{R_{i,t}} (1-x)^{n-R_{i,t}}}
    {\int_{0}^{1}x^{R_{i,t}} (1-x)^{n-R_{i,t}} dx}.
\]
The right-hand side is the density of a beta variable with parameters $(R_{i,t}+1,n-R_{i,t}+1)$. We conclude that 
\[
    \mathbb P^{x\sim \mu_{i,t}}\big[x\in [\ugamma,\hat p_{i,t}]\big]
    \geq
    (\uf/\of) \cdot \mathbb P^{z\sim \Beta(n-R_{i,t}+1,R_{i,t}+1)}\big[z\in [\ugamma,\hat p_{i,t}]\big]
\]
For $n$ sufficiently large, it follows from Lemma~\ref{lem:beta-clt} and \eqref{eq:pitbetaogamma} that
\[
    \mathbb P^{z\sim \Beta(n-R_{i,t}+1,R_{i,t}+1)}\big[z\in [\ugamma,\hat p_{i,t}]\big]\geq \frac{1}{3}.
\]
Therefore $\bbP^{\mu_{i,t}}[p_i\leq \hat p_{i,t}]\geq \frac{1}{3}$, proving \eqref{eq:good-chance-to-be-ok}. 
\end{proof}

\begin{lemma}
\label{lem:probmove}
Assume $\mu$ is $(\ugamma,\ogamma,\uf,\of)$-admissible and that \eqref{eq:pitbetaogamma} holds. For $n=n_{i,t}$, let $\tilde n\geq 1$ satisfy $|\tilde n-\varrho n|\leq 2$. Denote by
\[
    \tilde p_i 
    =
    \frac{R_{i,n+\tilde n}-R_{i,n}}{\tilde n}
\] 
the average reward from the $(n+1)$-th through $(n+\tilde n)$-th samples of arm $i$. Then as $n\to\infty$, for any sequence $\Delta_n=\Theta(1/\log n)$, 
\begin{equation}
\label{eq:probmove-lb}
    \mathbb P^t[\tilde p_i\leq \theta^{-1}(\theta(\hat p_{i,t})-\Delta_n)] \geq  \exp\lt(-\frac{n\varrho \Delta_n^2 (1+o_n(1))}{2}\rt).
\end{equation}
\end{lemma}

\begin{proof} 
Stochastic monotonicity implies that 
\[
    \mathbb P\left[\frac{\Bin(\tilde n,p)}{\tilde n}\leq \theta^{-1}\big(\theta(\hat p_{i,t})-\Delta_n\big)\right]
\]  
is a decreasing function of $p\in [0,1]$.
Combining with Lemma~\ref{lem:might-be-ok}, it follows that 
\begin{align*}
    \mathbb P^t[E] & =
     \int \mathbb P\left[\frac{\Bin(\tilde n,x)}{\tilde n}\leq \theta^{-1}\big(\theta(\hat p_{i,t})-\Delta_n\big)\right] d \mu_{i,t}(x)
     \\
    & \geq 
    \mathbb P^{\mu_{i,t}}[p_i\leq \hat p_{i,t}]
    \cdot
    \mathbb P\left[\frac{\Bin(\tilde n,\hat p_{i,t})}{\tilde n}\leq \theta^{-1}\big(\theta(\hat p_{i,t})-\Delta_n\big)\right]
    \\
    & \geq 
    \frac{\uf}{3\of}
    \cdot
    \mathbb P\left[\frac{\Bin(\tilde n,\hat p_{i,t})}{\tilde n}\leq \theta^{-1}\big(\theta(\hat p_{i,t})-\Delta_n\big)\right].
\end{align*}
Since $\theta$ is smooth with smooth inverse on $[\ugamma,\ogamma]$ and $\Delta_n\leq o_n(1)$, we have 
\begin{align*}
    \hat p_{i,t}-\theta^{-1}\big(\theta(\hat p_{i,t})-\Delta_n\big)
    &=
    (1\pm o_n(1)) \Delta_n
    \cdot(\theta^{-1})'\big(\theta(\hat p_{i,t})\big)
    \\
    &=
    \frac{(1\pm o_n(1))\cdot\Delta_n}{\theta'(\theta^{-1}(\hat p_{i,t}))}
    \\
    &=
    (1\pm o_n(1))
    \cdot
    \Delta_n
    \sqrt{\hat p_{i,t}(1-\hat p_{i,t})}.
\end{align*}
The result now follows from Lemma~\ref{lem:moderate}, where we absorb the factor $\uf/(3\of)$ into the $o_n(1)$. 
\end{proof}

\subsection{Proof of Theorem~\ref{thm:mainLB}}

Recall the definition~\eqref{eq:ogamma-def} of $\ugamma$ and $\ogamma$. We require $\cA$ to be $B$-batch-compressed for $B=B(N,\varrho)$ containing:
\begin{enumerate}
	\item All positive integers at most $N^{2\varrho}$.
	\item All positive multiples of $\lfloor N^{\varrho}\rfloor$ at most $N^{6\varrho}$.
	 \item Integers of the form $\lfloor N^{6\varrho}(1+\varrho)^j \rfloor$ for $j\geq 0$.
\end{enumerate}

It is easy to see that $B$ thus defined is $\varrho$-slowly increasing for any $\varrho>0$ and $N$ sufficiently large. We denote $b_k=\lfloor N^{6\varrho}(1+\varrho)^k \rfloor$ so that $|b_{k+1}-(1+\varrho)b_k|\leq 2$. (This choice of indexing differs from that of Definition~\ref{defn:batch}, which will not be used in the sequel.)

We next construct our randomness distorting adversary $\adv=\adv(N,\varrho)$. For each arm $i$, the adversary $\adv$ acts as follows depending on the current number of samples $n_{i,t}$.

\begin{enumerate}
    \item 
    \label{it:init-nothing}
    If $n_{i,t}\leq N^{2\varrho}$, then $\adv$ does nothing. 
    \item 
    \label{it:init-mild}
    When $N^{2\varrho}\leq n_{i,t}< N^{6\varrho}$ increases by $N^{\varrho}$, $\adv$ declares that the average reward of this batch of $N^{\varrho}$ samples is at most $\overline{\gamma}-N^{-\varrho}$.
    \item 
    When $n_{i,t}$ increases from $b_k\geq N^{6\varrho}$ to $b_{k+1}$:
    \begin{enumerate}[(a)]
        \item 
        \label{it:main-adversary}
            If $\hat p_i(b_k)> \beta$ holds, then $\adv$ declares that
            \begin{equation}
            \label{eq:main-decrement}
                \theta(\hat p_i(b_{k+1}))\leq \theta(\hat p_i(b_k))-\frac{\varrho(1+10\varrho)d_{F}(\alpha,\beta)}{\log N}.
            \end{equation}
        \item
        \label{it:beta-bad}
            If $\hat p_i(b_k)\leq \beta$ holds, then $\adv$ declares that 
            \[
                \hat p_i(b_{k+1})\leq\beta.
            \]
    \end{enumerate} 
    \item 
    \label{it:finish-bad}
    When the $\cA$ chooses the arm $a_{i^*}$ to output, $\adv$ declares that $p_{i^*}<\beta$.
\end{enumerate}

Due to step \ref{it:finish-bad}, the declarations made by $\adv$ ensure that $p_{i^*}<\beta$. Recalling Lemma~\ref{lem:adversary-abstract} and Proposition~\ref{prop:slowly-increasing}, it remains to show the upper bound
\[
    \strength(\adv)\leq \frac{(c_{\alpha,\beta}+C_*\varrho)N}{\log^2(N)}
\]
for a constant $C_*=C_*(\ugamma,\ogamma,\uf,\of,\beta,\alpha)$ independent of $\varrho$ (and $N$). 
We show this bound in several parts. Recalling \eqref{eq:cost}, we refer to the \emph{cost} of a step above as the contribution to $\Cost$ from the corresponding declarations by $\adv$. The most important parts are Lemmas~\ref{lem:cost-estimate-main} and \ref{lem:p-downward-drift}, which bound the cost of the main step~\ref{it:main-adversary} and form the dominant contribution to $\Cost$. Note that throughout the analysis below, all cost upper bounds hold almost surely and we \textbf{assume that all of $\adv$'s declarations hold true}.

\begin{lemma}
\label{lem:init-mild-cost}
The total cost from step~\ref{it:init-mild} is at most $C_* N^{1-\varrho}$, for $N\geq C(\ugamma,\ogamma,\uf,\of,\beta,\alpha,\varrho)$ sufficiently large.
\end{lemma}

\begin{proof} 
The probability for each such declaration by $\adv$ is at least
\begin{equation}
\label{eq:simple-binomial}
    \mathbb P[\Bin(N^{2\varrho},\ogamma)\leq \ogamma N^{2\varrho}-N^{\varrho}]
\end{equation}
since $p_i\leq \ogamma$ almost surely. Recall that a $\Bin(N^{2\varrho},\ogamma)$ random variable obeys a central limit theorem centered at $\ogamma N^{2\varrho}$ with standard deviation at least $C(\ogamma)N^{\varrho}$. Therefore the probability in \eqref{eq:simple-binomial} is at least $\frac{1}{3}$ for $N$ is sufficiently large depending on $\varrho$. Hence each such declaration costs at most $C_*$ for $N$ sufficiently large. Moreover such declarations can occur only $N^{1-\varrho}$ times because each one involves $N^{\varrho}$ samples, and the base algorithm $\cA$ is an $N$-sample algorithm. This completes the proof. 
\end{proof}

\begin{lemma}
\label{lem:beta-bad-cost}
The total cost from step~\ref{it:beta-bad} is at most $C_* N^{1-6\varrho}$ as long as $N\geq C(\ugamma,\ogamma,\uf,\of,\varrho)$.
\end{lemma}

\begin{proof} 
It suffices to show that the cost per step~\ref{it:beta-bad} declaration is at most $C_*$. This follows from \eqref{eq:good-chance-to-be-ok} and stochastic monotonicity. 
\end{proof}

\begin{lemma}
\label{lem:cost-estimate-main}
The total cost from step~\ref{it:main-adversary} is at most
\[
    \frac{N}{\log^2(N)}\cdot(c_{\alpha,\beta}+C_* \varrho+o_N(1)).
\]
\end{lemma}

\begin{proof} 
We claim that the cost from a single instance of step~\ref{it:main-adversary} when increasing from $b_k$ to $b_{k+1}$ samples is at most
\[
    \left(\frac{(b_{k+1}-b_k)}{\log^2(N)}\right)(c_{\alpha,\beta}+C_*\varrho+o_N(1)).
\]
This implies the desired result since $\cA_N$ is an $N$-sample algorithm. Taking $\Delta=(1+10\varrho)d_F(\alpha,\beta)/\log(N)$ in Lemma~\ref{lem:probmove}, we find that the declared event has probability at least
\[
    \exp\lt(-\frac{(b_{k+1}-b_k)(1+10\varrho)^2d_F(\alpha,\beta)^2(1+o_N(1))}{2 \log^2(N)}\rt)
    \geq 
    \exp\lt(-\frac{(b_{k+1}-b_k)}{ \log^2(N)}\big(c_{\alpha,\beta}+C_* \varrho+o_N(1)\big)\rt).
\]
This implies the desired claim and completes the proof. 
\end{proof}

\begin{lemma}
\label{lem:p-downward-drift}
For any $a_i$ sampled $b_0=\lfloor N^{6\varrho}\rfloor$ times, $\hat p_i(b_0)\leq \ogamma$.
\end{lemma}

\begin{proof} 
By definition of $\adv$, 
\begin{align*}
    \hat p_i(b_0)
    &\leq
    \frac{N^{2\varrho}+(N^{6\varrho}-N^{2\varrho})(\ogamma-N^{-\varrho})}{N^{6\varrho}}
    \\
    &=
    \ogamma 
    -\frac{1}{N^{\varrho}}
    +
    \frac{(1-\ogamma)}{N^{4\varrho}}
    +\frac{1}{N^{5\varrho}}
    \\
    &\leq \ogamma.
\end{align*}
In the last step we used the fact that 
\[
    \frac{1}{N^{\varrho}}
    \geq
    \frac{(1-\ogamma)}{N^{4\varrho}}
    +\frac{1}{N^{5\varrho}}
\]
for any $\varrho>0$ if $N$ is sufficiently large. 
\end{proof}

\begin{lemma}
\label{lem:beta-bad-analysis}
For $\varrho\in (0,1/100)$, if $n_{i,t}\geq N^{1-\varrho}$ and the declarations of $\adv$ hold, then $\hat p_{i,t}\leq \beta$.
\end{lemma}

\begin{proof} 
We analyze the rate at which the adversary forces $\theta(\hat p_i(b_k))$ to decrease. From \eqref{eq:main-decrement} and \eqref{lem:p-downward-drift} it follows that for $k$ with $b_k\geq N^{1-\varrho}$, we have
\begin{align*}
    \theta(\hat p_i(b_k))
    &\leq 
    \theta(\ogamma)
    -
    \frac{\varrho(1+10\varrho)d_{F}(\alpha,\beta)\log_{1+\varrho}(N^{1-8\varrho})}{\log N}
    \\
    &=
    \theta(\ogamma)
    -
    \frac{\varrho(1+10\varrho)(1-8\varrho)d_F(\alpha,\beta)}{\log(1+\varrho)}
    \\
    &\leq
    \theta(\ogamma)
    -
    (1+\varrho) d_F(\alpha,\beta)
    \\
    &\stackrel{\eqref{eq:ogamma-def}}{<}
    \theta(\beta).
\end{align*}
Here we used the fact that $\log(1+\varrho)\leq \varrho$ and $(1+10\varrho)(1-8\varrho)\geq 1$ for $\varrho\in (0,1/100)$. Since $\theta$ is increasing, this shows that $\hat p_{i,t}=\hat p_i(b_k)< \beta$ for $b_k\geq N^{1-\varrho}$, completing the proof.
\end{proof}

\begin{lemma}
\label{lem:rejection-cost}
The cost from step~\ref{it:finish-bad} is at most $C_* \big(N^{1-\varrho}+1\big)$.
\end{lemma}

\begin{proof} 
First, if $\hat p_{i^*,N}\leq \beta$ then the cost from step~\ref{it:finish-bad} is at most $C_*$. On the other hand if $\hat p_{i^*,N}> \beta$, then Lemma~\ref{lem:p-downward-drift} implies $n_{i^*,N}\leq N^{1-\varrho}$. Since the prior $\mu$ is supported in $[\ugamma,\ogamma]$, the likelihood ratio of updates from $N^{1-\varrho}$ samples is almost surely bounded by $e^{C_* N^{1-\varrho}}$. Therefore
\begin{align*}
    \mathbb P^{x\sim\mu_{i,N}}[x<\beta]
    &\geq
    e^{-C_* N^{1-\varrho}}
    \mathbb P^{x\sim\mu}[x<\beta]
    \\
    &\geq
    e^{- C_* N^{1-\varrho}} \frac{(\beta-\ugamma)\uf}{\of}.
\end{align*}
This completes the proof. 
\end{proof}

We now combine the lemmas above to conclude Theorem~\ref{thm:main} via \eqref{eq:LB-varrho}.

\begin{proof}[Proof of Theorem~\ref{thm:main}]
Let $C_*'$ be a larger constant depending on the same parameters. Then by Lemmas~\ref{lem:init-mild-cost}, \ref{lem:beta-bad-cost}, and \ref{lem:rejection-cost}, the total cost from Steps~\ref{it:init-mild}, \ref{it:beta-bad}, \ref{it:finish-bad} combines to $C_*' N^{1-\varrho})\leq o_N(N/\log^2 N)$. The main cost contribution of 
\[
    \frac{N}{\log^2 N}(c_{\alpha,\beta}+C_*\varrho+o_N(1)).
\]
comes from Lemma~\ref{lem:cost-estimate-main}, and all other terms are of strictly smaller order.
We have thus constructed a reservoir sequence $(\mu_N(\varrho))_{N\geq 1}$ satisfying \eqref{eq:LB-varrho} for arbitrary $\varrho>0$, completing the proof. 
\end{proof}

\section{An Optimal Algorithm with Fixed Budget}
\label{sec:budget}

Here we provide an asymptotically optimal algorithm which establishes Theorems~\ref{thm:alg-1}, \ref{thm:alg-2}, and \ref{thm:alg-3}. In the next subsection in which we show how to reduce the other results mentioned to Theorem~\ref{thm:main} (in which $\alpha$ is given) using Proposition~\ref{prop:alpha}. Our main focus will then be to prove Theorem~\ref{thm:main}.

We will fix $\varrho>0$ small and construct a sequence of $N$-sample algorithms $(\cA(N,\varrho))$ satisfying the slightly relaxed guarantee
\begin{equation}
\label{eq:UB-varrho}
    \liminf_{N\to\infty} \frac{(-\log(\bbP^{\mu_N(\varrho)}[p_{i^*}< \beta]))\cdot \log^2 N}{N}
        \geq 
        c_{\alpha,\beta}-\lambda(\varrho)
\end{equation}
for a (possibly different) function $\lambda$ satisfying $\lim_{\varrho\to 0}\lambda(\varrho)=0$ (for fixed $\alpha,\beta,\eta$). Here $(\mu_N)_{N\geq 1}$ is any sequence of reservoir distributions satisfying $G_{\mu_N}^{-1}(1-\eta)=\alpha$. An elementary diagonalization argument then implies Theorem~\ref{thm:main}. Thus it suffices to construct algorithms satisfying \eqref{eq:UB-varrho} for any desired $\varrho>0$.

\subsection{Reduction to Known $\alpha$}
\label{subsec:reduction}

We explain why Theorems~\ref{thm:alg-1}, \ref{thm:alg-2}, and \ref{thm:alg-3} all follow from Theorem~\ref{thm:main} (more precisely, the uniform statement given in Remark~\ref{rem:main-unif}). We begin with Theorem~\ref{thm:alg-1}, where 
\[
    \alpha_N=\frac{1}{\eta_1-\eta_2}\cdot \int_{1-\eta_1}^{1-\eta_2}
    G_{\mu_N}^{-1}(x) dx.
\]
Let $J=\lceil \frac{6}{\eps(\eta_1-\eta_2)}\rceil$ and define
\[
    \eta^{(j)}=\frac{(J-j)\eta_1+j\eta_2}{J},\quad j\in [J].
\]
It is easy to see that $\eta^{(j+1)}-\eta^{(j)}\leq \eta^{(j)}$ for all $j$.
We next apply Alg.~\ref{alg:alpha} on $(\eta^{(j)},\eta^{(j+1)}-\eta^{(j)},\eps',\delta')$ for $0\leq j\leq J-1$, with:
\begin{align*}
    \eps'& = \log^{-1/3}(N),
    \\
    \delta' &= e^{-\frac{10N}{\log^2(N)}}/J.
\end{align*} 
This requires sample complexity 
\begin{equation}
\label{eq:NA}
    N_A\leq \frac{C(\eta_1,\eta_2) N \log\log(N)}{\log(N)}\leq o_N(N).
\end{equation}
Let $\hat\alpha_j$ be the resulting output. With probability $1-J\delta$, we have for each $0\leq j\leq J-1$,
\begin{equation}
\label{eq:alphaj}
    \hat\alpha_j \in \lt[G^{-1}(1-\eta^{(j)})-\frac{\eps}{3},G^{-1}\lt(1-\eta^{(j+1)}\rt)+\frac{\eps}{3}\rt].
\end{equation}
Note that the function $G^{-1}_{\mu}$ is increasing and $[0,1]$-valued. Therefore if \eqref{eq:alphaj} holds for each $j$, then
\[
    \lt|\frac{1}{J}\cdot\sum_{j=0}^{J-1} \hat\alpha_j 
    -
    \frac{1}{\eta_1-\eta_2}\cdot \int_{1-\eta_1}^{1-\eta_2}
    G_{\mu_N}^{-1}(x) dx 
    \rt|
    \leq
    \frac{\eps}{3} + \frac{1}{J} \leq \frac{\eps}{2}.
\]
Therefore the estimator
\[
    \hat\alpha_A=\frac{1}{J}\cdot\sum_{j=0}^{J-1} \hat\alpha_j
\]
satisfies 
\[
    \bbP\lt[\lt|\hat\alpha_A
    -
    \frac{1}{\eta_1-\eta_2}\cdot \int_{1-\eta_1}^{1-\eta_2}
    G_{\mu_N}^{-1}(x) dx 
    \rt|
    \leq \eps/2\rt]
    \geq
    1-J\delta' 
    =
    1-e^{-\frac{10N}{\log^2(N)}}.
\]
Finally, $c_{\alpha,\alpha-\eps}\leq \pi<10$ for any $\alpha,\eps\in [0,1]$ (see \eqref{eq:theta}). Therefore the $\delta'=e^{-\frac{10N}{\log^2(N)}}$ failure probability above has a negligible contribution in Theorem~\ref{thm:alg-1}. It follows that applying Theorem~\ref{thm:main} with $\alpha=\hat\alpha_A$ as above and $N'=N-N_A$ implies Theorem~\ref{thm:alg-1}.

We now turn to Theorem~\ref{thm:alg-2}, where $\mu_N$ is required to satisfy $G_{\mu_N}^{-1}(1-\eta)\geq \frac{1+\eps}{2}$. We run Alg.~\ref{alg:alpha} with parameters 
\begin{align*}
    \eta_1&=\eta,
    \\
    \eta_2&=\log^{-1/3}(N),
    \\
    \eps'& = \log^{-1/3}(N),
    \\
    \delta' &= e^{-\frac{10N}{\log^2(N)}}
    .
\end{align*} 
The sample complexity $N_B$ again satisfies $N_B\leq o(N)$ exactly as in \eqref{eq:NA}.
Let $\hat \alpha_B+\eps'$ be the resulting output.
Then with probability at least $1-e^{-\frac{10N}{\log^2(N)}}$, 
\[
    \hat\alpha_B\geq G_{\mu_N}^{-1}(1-\eta)-2\eps'
\]
and so with $\eps''=\eps-2\eps'$, we have
\[
     \hat\alpha_B-\eps''\geq G_{\mu_N}^{-1}(1-\eta)-\eps.
\]
Moreover, also with probability at least $1-e^{-\frac{10N}{\log^2(N)}}$,
\[
    \hat\alpha_B\leq G_{\mu_N}^{-1}(1-\eta+\eta_2).
\]
It follows that applying the algorithm of Theorem~\ref{thm:main} with 
\[
  (N,\alpha,\eta,\eps) = (N-N_B,\hat\alpha_B, \eta-\eta_2,\eps-2\eps')
\]
suffices to recover Theorem~\ref{thm:alg-2}, since $\eta_2$ and $\eps'$ tend to $0$ as $N\to\infty$. As in our discussion of Theorem~\ref{thm:alg-1} above, the failure probability $e^{-\frac{10N}{\log^2(N)}}$ is negligible compared to the relevant rate in Theorem~\ref{thm:alg-2}.

Finally, Theorem~\ref{thm:alg-3} relies on the simple fact
\begin{equation}
\label{eq:quantile-limit}
    \lim_{\eta\to 0} G_{\mu}^{-1}(1-\eta)=\mu^.
\end{equation}
Recall that $\mu^*\in [0,1]$ denotes the maximum value in the support of $\mu$. We run Alg.~\ref{alg:alpha} on $(\eta_1,\eta_2,\eps',\delta')$ where:
\begin{align*}
    \eta_1&= \log^{-1/3}(N),
    \\
    \eta_2&=\eta_1/2,
    \\
    \eps'& = \eps_1-\eps,
    \\
    \delta' &= e^{-\frac{10N}{\log^2(N)}}
    .
\end{align*} 
It follows from Proposition~\ref{prop:alpha} that the resulting output $\hat\alpha_C+\frac{\eps_1-\eps}{2}$ is computed using $O\lt(\frac{N \log\log(N)}{\log(N)}\rt)\leq o(N)$ samples as in the previous cases. Moreover for $N$ sufficiently large:
\begin{align*}
    \bbP\lt[\hat\alpha_C+\frac{\eps_1-\eps}{2} 
    \geq 
    \mu^*-\frac{\eps'}{3}-o_N(1)\rt]
    &
    \stackrel{\eqref{eq:quantile-limit}}
    {\geq}
    \bbP\lt[\hat\alpha_C +\frac{\eps_1-\eps}{2}
    \geq 
    G_{\mu}^{-1}(1-\eta_1)-\frac{\eps'}{3}\rt]
    \\
    &\geq
    1-\delta'
    \\
    &=
    1-e^{-\frac{10N}{\log^2(N)}}.
\end{align*}
Since $\eps_1>\eps$, this means for $N\geq N_0(\mu,c',\dots)$ large enough,
\[
    \bbP\lt[\hat\alpha_C \geq \mu^*-(\eps_1-\eps)\rt]
    \geq 
    1-e^{-\frac{10N}{\log^2(N)}}.
\]
Note that Alg.~\ref{alg:alpha} also ensures that with probability $1-e^{-\frac{10N}{\log^2(N)}}$, 
\begin{align*}
    \hat\alpha_C 
    \leq 
    \mu^* + \frac{\eps'}{3}-\frac{\eps_1-\eps}{2}
    &=
    \mu^*-\frac{\eps_1-\eps}{6}
    \\
    &\leq
    G_{\mu}^{-1}(1-\eta')
\end{align*}
for some $\eta'(\mu,\eps_1,\eps)>0$. It follows that applying Theorem~\ref{thm:main} with 
\[
    (N,\alpha,\eta,\eps)=(N-N',\hat\alpha_C,\eta',\eps)
\]
implies Theorem~\ref{thm:alg-3}.

\subsection{The Fixed Budget Algorithm}

We now present Algorithm~\ref{alg:main} for the fixed budget problem. Algorithm~\ref{alg:main} studies one arm $a_i$ at a time, moving to $a_{i+1}$ if $a_i$ is rejected. Similarly to the previous section, some details are needed while $n_{t,i}$ is small, since large deviation asymptotics may not have kicked in yet. As explained at the start of the section, we choose a small constant $\varrho>0$. In fact, we will eventually choose small constants 
\[
0<\varrho\ll\varrho_1\ll\varrho_2\ll\varrho_3\ll\varrho_4\ll\varrho_5\ll 1
\]
which all tend to $0$ as $\varrho\to 0$. These constants will be defined throughout the proof.
More formally, these values can be obtained by choosing $\varrho_5>0$ arbitrarily small, then $\varrho_4>0$ sufficiently small depending on $\varrho_5$, and so on.

Algorithm~\ref{alg:main} operates in a batch-compressed way, for a sequence $(b_1,b_2,\dots)$ defined as follows:
\begin{align*}
    b_0&=\lceil \varrho_1\log^2(N)\rceil,
    \\
    k_0&=\lceil \log_{1+\varrho}\left( \log^4(N)/b_0 \right)\rceil
    \\
    b_k&=b_0 (1+\varrho)^{k}, \quad k\leq k_0
    \\
    b_{k_0+j}&=\lceil (1+\varrho)^j b_{k_0}\rceil,\quad j\geq 1
    \\
    \tau_k&= \alpha-\varrho-\frac{k}{\sqrt{\log N}},\quad k\leq k_0
    \\
    \tau_{k_0+j}&= \theta(\alpha-2\varrho) -j\cdot \frac{d_F(\alpha,\beta)\varrho (1-\varrho_2)}{\log N},\quad j\geq 1.
\end{align*}
Note in particular that $b_{k_0}\geq \log^4(N)$. We denote by $\hat p_{i,t}$ the empirical average reward collected by $a_i$ from its first $t$ samples.

\begin{figure}[!h]
\SetKwFor{Loop}{loop}{}{end}
\begin{algorithm2e}[H]
\label{alg:main}
\caption{Output arm with $p_i\geq \beta$ using $N$ samples with high probability}

\SetAlgoLined\DontPrintSemicolon

\textbf{input: }an infinite sequence of arms $i=1,2,\dots$
\\
\text{initialize: }$i=0$
\\
\While{fewer than $N$ samples have been collected}{
    $i\leftarrow i+1$
    \\
    Collect $b_0$ samples of arm $i$.
    \\
    \If{$\hat p_{i,b_0}\leq \alpha-\varrho$}{
        \textbf{Reject} arm $i$
    }
    \For{$k=1,2,\dots,k_0$}{
        Collect $b_k-b_{k-1}$ samples of arm $i$ for a total of $b_k$ samples.
        \\
        \If{$\hat p_{i,b_k}\leq \alpha-\varrho-\frac{k}{\sqrt{\log N}}$}{
            \textbf{Reject} arm $i$; 
        }  
    }
    \For{$j=1,2,\dots$}{
        Collect $b_{k_0+j}-b_{k_0+j-1}$ samples of arm $i$ for a total of $b_{k_0+j}$.
        \\
        \If{$\theta(\hat p_{i,b_{k_0+j}})\leq \theta(\alpha-2\varrho) -j\cdot \frac{d_F(\alpha,\beta)\varrho (1-\varrho_2)}{\log N}$}
        {
        \textbf{Reject} arm $i$
        }
    }
}
Return arm $i$.
\end{algorithm2e}
\end{figure}

The role of the values $b_j$ is as follows. When an arm $a_i$ reaches $b_k$ samples for some $k\geq 0$, it is checked for possible rejection by comparing its empirical average reward  to the threshold $\tau_k$. 
Algorithm~\ref{alg:main} rejects arm $i$ and moves to arm $a_{i+1}$ if the empirical average $\hat p_{i,b_k}$ of arm $a_i$ drops below a moving threshold $\tau_k$. The threshold $\tau_k$ begins close to $\alpha$ and gradually decreases until reaching $\beta+\varrho$ by the time $\tau_k\geq \Omega(N)$.

So far, our informal description of Alg.~\ref{alg:main} also applies to the algorithm proposed in \cite{grossman2016amplification}. We now highlight two important differences. The first is that our algorithm is defined more carefully during the ``early" phases when an arm has been sampled at most $N^{O(\varrho)}$ times. This is crucial for carrying out a rigorous analysis. The second difference is that in the main phase, we increase the sample size for a given arm in powers of $1+\varrho$ rather than powers of $2$, and also move the rejection thresholds $\tau_k$ based on the Fisher information distance via the function $\theta$. The latter ingredients allow us to obtain the optimal constant factor.

Our analysis of Alg \ref{alg:main} begins with the following lemma. 

\begin{lemma}
\label{lem:doob-maximal}
Suppose $(Y_i)_{i\geq 1}$ are \iid random variables with non-negative integer values, and $\mathbb E[Y_i^c]\leq 1$ holds for some constant $c\geq 0$. Let 
\[
    M=\sup_{j\geq 0}\prod_{1\leq i\leq j} Y_i.
\]
Then
\[
    \mathbb P[M\geq A]
    \leq 
    A^{-c}.
\]
\end{lemma}

\begin{proof}
Let $M_j=\prod_{1\leq i\leq j} Y_i$ and observe that $M_j^c$ is a positive supermartingale with $M_0=0$. The result follows by Doob's maximal inequality.
\end{proof}

We will apply Lemma~\ref{lem:doob-maximal} in the following way. Let $X_i$ be the number of samples used by arm $a_i$ before rejection, and $I_i\in\{0,1\}$ be the indicator of the event that $a_i$ is ever rejected, even if Algorithm~\ref{alg:main} were to continue past time $N$ and sample arm $i$ an infinite number of times. We set 
\[
    Y_i=e^{X_i}\cdot I_i,
\]
With $M$ defined from $(Y_i)_{i\geq 1}$ as in Lemma~\ref{lem:doob-maximal}, it follows that $\log(M)$ is at most the amount of time spent on eventual rejections before the first eventually accepted arm. Therefore if $\log(M)\leq N(1-\varrho)$, we conclude that the last arm to be studied was sampled at least $N\varrho$ times. Since it was not rejected during that time, we can conclude this arm has $p_i\ge \beta$ with probability $1-e^{-\Omega_{\varrho}(N)}$. The main contribution to the failure probability of Algorithm~\ref{alg:main} comes from the event $\{M\geq A\}$ above, for suitable $A$. Correspondingly, the main work will be to verify $\bbE[Y_i^c]\leq 1$ for suitable $c$.

Note that $Y_i\in \{0\}\cup [1,\infty)$ almost surely for each $i$. Therefore a necessary first step in showing $\bbE[Y_i^c]\leq 1$ is to lower bound $\bbP[Y_i=0]$, the probability that Algorithm~\ref{alg:main} never rejects $a_i$. We now give a sufficient lower bound from the event $p_i\geq \alpha$.

\begin{proposition}
\label{prop:chance-to-accept}
Let $x_1,x_2,\dots$ be an \iid Bernoulli$(p)$ sequence for $p\geq \alpha$, and let $S_k=\sum_{i=1}^k x_i$ and set 
\[
    \uS=\inf_{k\geq 1} S_k/k.
\]
Then $\uS\geq \alpha-\varrho$ holds with probability at least $c(\alpha,\varrho)>0$. Thus $\mathbb E[I_i]\leq 1-c(\alpha,\varrho)$.
\end{proposition}

\begin{proof}
Since the probability that $\uS\geq \alpha-\varrho$ is increasing in $p$ it suffices to take $p=\alpha$ and show the probability is positive for any $\varrho>0$. Assume not. Then by restarting the indexing every time $S_k\leq k(\alpha-\varrho)$ holds, we find that 
\[
    \lim\inf_{n\to\infty} S_n/n \leq \alpha-\varrho.
\]
This contradicts the strong law of large numbers, thus completing the proof of the first assertion. The second assertion follows since if $S_k/k\geq \alpha-\varrho$ for all $k$ where $x_1,\dots$ are the rewards of arm $i$, then arm $i$ will never be rejected by Algorithm~\ref{alg:main}.
\end{proof}

Based on Proposition~\ref{prop:chance-to-accept} above, to show 
\[
    \mathbb E\left[e^{X_i\cdot \frac{c_{\alpha,\beta}-\varrho_3}{\log^2 N}}\cdot I_i\right]\leq 1
\]
(which is essentially what we want in light of Lemma~\ref{lem:doob-maximal}), it suffices to show that 
\begin{equation}
\label{eq:Xi-Ri-tail-bound}
    \mathbb E\left[\left(e^{X_i\cdot \frac{c_{\alpha,\beta}-\varrho_3}{\log^2 N}}-1\right)\cdot I_i\right]\leq c(\alpha,\varrho).
\end{equation}
We let $I_i^t=I_i\cdot 1_{X_i=t}$ be the event that arm $i$ was rejected after exactly $t$ steps. Since Alg \ref{alg:main} can only reject after $b_j$ samples, we have 
\[
    I_i=\sum_{j=0}^{\infty} I_i^{b_j}
\]
We use this to break the left-hand side of \eqref{eq:Xi-Ri-tail-bound} into three separate parts and estimate the parts separately. The parts correspond to $b_0$, $b_1$ through $b_{k_0}$, and $b_{k_0+1}$ onward. The first two parts are easier and handled in Subsection~\ref{subsec:alg-initial} below. The final term is the main contribution and is handled in Subsection~\ref{subsec:large-sample}.

\subsection{Analysis of Algorithm~\ref{alg:main} in the Small and Medium Sample Phases}
\label{subsec:alg-initial}

Proposition \ref{prop:t0-bound} bounds the contribution to \eqref{eq:Xi-Ri-tail-bound} from the \emph{small sample phase}, i.e. the first rejection condition in line 7 of Alg \ref{alg:main}.

\begin{proposition}
\label{prop:t0-bound}
For any $\alpha,\varrho$ there is $\varrho_1>0$ sufficiently small that with $b_0$ as defined above, and with $N$ sufficiently large,
\[
    \mathbb E\left[\left(e^{X_i\cdot \frac{c_{\alpha,\beta}-\varrho_3}{\log^2 N}}-1\right)\cdot I_i^{b_0}\right]
    \leq
    c(\alpha,\varrho)/4
\]
\end{proposition}

\begin{proof}
    It suffices to observe that for fixed $\alpha,\varrho$ and $\varrho_1$ small and $N$ sufficiently large, we have
    \[
        e^{b_0\cdot \frac{c_{\alpha,\beta}-\varrho_3}{\log^2 N}}-1
        \leq 
        e^{\varrho_1}-1
        \leq 2\varrho_1.
    \]
\end{proof}

Proposition \ref{prop:k0-tail-bound} bounds the contribution to \eqref{eq:Xi-Ri-tail-bound} from the \emph{medium sample phase}, i.e. the second rejection condition in line 12 of Alg \ref{alg:main}.

\begin{proposition}
\label{prop:k0-tail-bound}
For any $\alpha,\varrho,\varrho_1$ and for $N$ sufficiently large,
\[
    \sum_{k=1}^{k_0}\mathbb E\left[\left(e^{X_i\cdot \frac{c_{\alpha,\beta}-\varrho_3}{\log^2 N}}-1\right)\cdot I_i^{b_k}\right]
    \leq
    c(\alpha,\varrho)/4
\]
\end{proposition}

\begin{proof}
    The event $I_i^{b_k}$ requires $|\hat p_{i,b_k}-\hat p_{i,b_{k-1}}|\geq \frac{1}{\sqrt{\log N}}$. Hence by a standard Chernoff estimate, regardless of the true reward probability $p_i$,
    \[
        \mathbb E[I_i^{b_k}]\leq e^{-\Omega_{\alpha,\varrho,\varrho_1}(b_k/\log N)}.
    \]
     Since by construction $b_0\geq \varrho_1\log^2 N$, we have
     \begin{align*}
     \mathbb E\left[\left(e^{X_i\cdot \frac{c_{\alpha,\beta}-\varrho_3}{\log^2 N}}-1\right)\cdot I_i^{b_k}\right]
     &
     \leq
     e^{b_k\frac{c_{\alpha,\beta}-\varrho_3}{\log^2 N}-\Omega_{\alpha,\varrho,\varrho_1}(b_k/\log N)}
     \\
     &\leq
     e^{-\Omega_{\alpha,\varrho,\varrho_1}(\log N)}
     \\
     &=
     N^{-\Omega_{\alpha,\varrho,\varrho_1}(1)}.
     \end{align*}
     Since $k_0\leq O(\log N)$, summing gives the desired conclusion.
\end{proof}

Propositions \ref{prop:t0-bound} and \ref{prop:k0-tail-bound} imply that the total contribution from rejections in the small and medium sample phases is at most $c(\alpha,\varrho)/2$. It remains to analyze the large sample phase in the following subsection.

\subsection{Analysis of Algorithm~\ref{alg:main} in the Large Sample Phase}
\label{subsec:large-sample}

Similarly to the previous section, the main part of the analysis concerns the large sample phases $b_{k_0+j}$ for $j\geq 1$. Our goal is to precisely estimate the rejection probability at each time $b_{k_0+j}$. Note that these estimates should not depend on the true average rewards $p_i$.

Our approach is based on exchangeability and avoids any consideration of $p_i$. For a given value $j$ and a large constant $L=L(\varrho)$, consider the sequence of times
\[
    b_{k_0+j-L},\, b_{k_0+j-L+1},\, \dots,\,b_{k_0+j}
\]
and the associated sequence of empirical average rewards
\begin{equation}
\label{eq:hatp_sequence}
    \hat p_{i,b_{k_0+j-L}},~
    \hat p_{i,b_{k_0+j-L+1}},~\dots,~
    \hat p_{i,b_{k_0+j}}.  
\end{equation}
It follows from the algorithm description that for $I_i^{b_{k_0+j}}$ to occur, we must have
\begin{equation}
\label{eq:hat-p-change}
    \hat p_{i,b_{k_0+j}}-\hat p_{i,b_{k_0+j-\ell}}
    \geq \ell\cdot \frac{d_F(\alpha,\beta)\varrho(1-\varrho_2)}{\log N},
    \quad 
    \forall~1\leq \ell\leq L.
\end{equation}
This is clear for $j>L$, but it holds also for $0\leq j\leq L$ as for $N$ sufficiently large,
\[
    \alpha-\varrho-\frac{k_0}{\sqrt{\log N}}
    -L\cdot \frac{d_F(\alpha,\beta)\varrho(1-\varrho_2)}{\log N}
    \geq 
    \alpha-2\varrho.
\]
By exchangeability, conditioned on the future values $\hat p_{i,b_{k_0+j}},\dots,\hat p_{i,b_{k_0+j-\ell}}$ the law of $\hat p_{i,b_{k_0+j-\ell-1}}$ depends only on $\hat p_{i,b_{k_0+j-\ell}}$ and is given explicitly by a hypergeometric variable. Recalling that $R_{i,t}=n_{i,t} \hat p_{i,t}$ is the total reward from the first $n_{i,t}$ samples of arm $i$, $R_{i,b_{k_0+j-\ell-1}}$ has hypergeometric conditional law given by:
\begin{align}
\label{eq:HG-conditional}
    \bbP\Big[
    R_{i,b_{k_0+j-\ell-1}}
    =
    k
    ~\big|~
    \big(\hat p_{i,b_{k_0+j}},\dots,\hat p_{i,b_{k_0+j-\ell}}\big)\Big]
\nonumber
    &=
    \bbP\big[
    R_{i,b_{k_0+j-\ell-1}}
    =
    k
    ~|~
    \hat p_{i,b_{k_0+j-\ell}}\big]
    \\
    &=
    \frac{\binom{b_{k_0+j-\ell-1}}{k}\binom{b_{k_0+j-\ell}-b_{k_0+j-\ell-1}}{R_{k_0+j-\ell}-k}}{\binom{b_{k_0+j-\ell}}{R_{k_0+j-\ell}}}\,.
\end{align}
We will refer to this as the $\HG\big(b_{k_0+j-\ell},b_{k_0+j-\ell-1},R_{k_0+j-\ell}\big)$ distribution.
Importantly, this distribution is independent of $\mu$.
We exploit this below to control the probability of a given sequence $\big(\hat p_{i,b_{k_0+j-L}},~\hat p_{i,b_{k_0+j-L+1}},~\dots,~\hat p_{i,b_{k_0+j}}\big)$ of empirical average rewards. The following useful result states that hypergeometric variables automatically inherit tail bounds from the corresponding binomial random variables.

\begin{lemma}[{\cite{luh2014large,hoeffding1994probability}}]
\label{lem:convex-dom-HG}
Fix non-negative integers $A\geq B,C$ and let $X\sim\HG(A,B,C)$ and $Y\sim \Bin(B,C/A)$.
Then for any convex function $f:\bbR\to\bbR$,
\[
    \mathbb E[f(X)]\leq\mathbb E[f(Y)]. 
\]
\end{lemma}

\begin{lemma}
\label{lem:HG-moderate}
For any $0<\uq<\oq<1$ and constants $\varrho>0$ there exists $\Delta_0(\uq,\oq,\varrho)$ and $N_0(\uq,\oq,\varrho)$ such that the following holds for all $p\in [\uq,\oq]$. For $n\geq n_0$ sufficiently large and $\frac{1}{\Delta_0 \sqrt{n}}\leq \Delta\leq \Delta_0$,
\[
    \mathbb P\left[\frac{\HG(n(1+\varrho),n,np(1+\varrho))}{n}
    \leq p- \Delta
    \right] 
    \leq 
    e^{\left(-\frac{\Delta^2}{2p(1-p)}+\varrho\right)n}.  
\]
\end{lemma}

\begin{proof}
     The corresponding binomial result Lemma~\ref{lem:moderate} is proved in \cite[Theorem 2.2]{de1992moderate} by upper bounding an exponential moment. The same proof applies here by Lemma~\ref{lem:convex-dom-HG}.
\end{proof}

It will be convenient to define a restricted set of \emph{good} sequences $(q_L,q_{L-1},~\dots,q_0)$. These satisfy the key properties of empirical average reward sequences \eqref{eq:hatp_sequence} for which $I_i^{b_{k_0+j}}$ holds. We say such a length $L+1$ sequence is good if the following conditions are satisfied:
\begin{enumerate}
    \item $q_0\in [\uq,\oq]\subseteq (0,1)$ for constants $0<\uq<\oq<1$ depending only on $\varrho, L$.
    \item 
    \begin{equation}
    \label{eq:2nd-condition}
    \max_{\ell_1,\ell_2} ~|q_{\ell_1}-q_{\ell_2}|\leq  O\big(1/\sqrt{\log N}\big).
    \end{equation}
    \item For each $1\leq \ell\leq L$:
    \begin{align*}
        \theta(q_0)
        &\leq 
        \theta(\alpha-2\varrho) -j\cdot \frac{d_F(\alpha,\beta)\varrho(1-\varrho_2)}{\log N}
        \\
        &\leq
        \theta(\alpha-2\varrho) -(j-\ell)\cdot \frac{d_F(\alpha,\beta)\varrho(1-\varrho_2)}{\log N}
        \\
        &\leq
        \theta(q_{\ell}).
    \end{align*}
\end{enumerate}

The third condition above is necessary for $I_i^{b_{k_0+j},i}=1$, and these together imply the first condition. Indeed for fixed $\uq,\oq$ and small $\varrho\in (0,1/10)$ one always has
\[
    \frac{\hat p_{i,b_{k_0+j-1}}}{\hat p_{i,b_{k_0+j}}}, \frac{1-\hat p_{i,b_{k_0+j-1}}}{1-\hat p_{i,b_{k_0+j}}} 
    \in 
    \big[1- 2\varrho, (1-2\varrho)^{-1}\big]
\]
for large enough $N$ and any $j$. Hence it suffices to take $\uq=\beta (1-2\varrho)^L$ and  $\oq=1-(1-\alpha) (1-2\varrho)^{L}$. With this choice, if
\[
    \hat p_{i,b_{k_0+j-L}},~
    \hat p_{i,b_{k_0+j-L+1}},~\dots,~
    \hat p_{i,b_{k_0+j}}.  
\]
is \textbf{not} good and $I_i^{b_{k_0+j}}=1$, then the second condition must be the only violated one. The following easy lemma controls the failure probability of the second condition. Recall from \eqref{eq:HG-conditional} that conditioning on $\hat p_{i,b_{k_0+j}}$ determines the joint conditional law of the previous conditional rewards, regardless of $\mu$.

\begin{lemma}
\label{lem:bad-seq-bound}
    All sequences violating only the second condition \eqref{eq:2nd-condition} above have probability at most 
    \[
        e^{-\Omega_{L,\varrho}(b_{k_0+j}/\log N)},
    \]
    even after conditioning on an arbitrary value for $\hat p_{i,b_{k_0+j}}$.
\end{lemma}

\begin{proof}

The claim follows by an elementary Chernoff estimate for hypergeometric variables, which hold just as for binomial variables by Lemma~\ref{lem:convex-dom-HG}. Indeed the assumption implies that some adjacent difference $|\hat p_{i,b_{k_0+j-\ell}}-\hat p_{i,b_{k_0+j-\ell+1}}|$ has size $\Omega(1/\sqrt{\log N})$. 
(Note for applying the Chernoff bound that $L$ is a constant independent of $N$, and so $b_{k_0+j-L}\geq \Omega_{L,\varrho}(b_{k_0+j}).$)
\end{proof}

We now focus on upper-bounding the probability of any good sequence $(q_L,\dots,q_0)$ appearing, conditionally on $q_0$.

\begin{lemma}
\label{lem:general-good-seq-bound}
For any good sequence $(q_L,q_{L-1},\dots,q_0)$ and $j\geq 0$,
\begin{align*}
&\bbP\Big[ 
   \big(\hat p_{i,b_{k_0+j-L}},~
    \hat p_{i,b_{k_0+j-L+1}},~\dots,~
    \hat p_{i,b_{k_0+j}}\big)=\big(q_L,q_{L-1},\dots,q_0\big)
    ~\big|~p_{i,b_{k_0+j}}=q_0
    \Big]
    \\
    &\leq 
    \exp\left(-\frac{(1-O(\varrho))}{2q_0(1-q_0)\varrho}\sum_{\ell=0}^{L-1} 
    b_{k_0+j-\ell}(q_{\ell}-q_{\ell+1})^2
    \right).
\end{align*}
\end{lemma}

\begin{proof}
It suffices to show that
\[
    \bbP[\hat p_{i,b_{k_0+j-\ell-1}}=q_{\ell+1}~|~q_{\ell}]
    \leq
    \exp\left(-\frac{(1-O(\varrho))}{2q_0(1-q_0)\varrho} 
    b_{k_0+j-\ell}(q_{\ell}-q_{\ell+1})^2
    \right)
\]
This follows by applying Lemma~\ref{lem:HG-moderate} to the hypergeometric random variable 
\[
    \hat p_{i,b_{k_0+j-\ell}}\cdot b_{k_0+j-\ell} - \hat p_{i,b_{k_0+j-\ell-1}}\cdot b_{k_0+j-\ell-1}
    =
    R_{b_{k_0+j-\ell}}-R_{b_{k_0+j-\ell-1}}.
\]
The fact that 
\[
    b_{k_0+j-\ell+1}-b_{k_0+j-\ell}=\varrho\cdot b_{k_0+j-\ell}\pm O(1)
\]
leads to the factor of $\varrho$ in the denominator of the desired result.
\end{proof}

\begin{lemma}
\label{lem:ql-diff}
For fixed problem parameters and $N$ large, any good sequence $(q_L,\dots,q_0)$ satisfies
\[
    q_{\ell}
    \geq
    q_0
    +
    \frac{\ell\cdot d_F(\alpha,\beta)\varrho(1-2\varrho_2)\cdot \sqrt{q_0(1-q_0)}}
    {(\log N)}
\]
\end{lemma}

\begin{proof}
    Recall that $\theta'(q)=\frac{1}{\sqrt{q(1-q)}}$ and that $\theta$ is smooth on $[\uq,\oq]\subseteq (0,1)$. By Item $2$ above, all $q_{\ell}$ are within $o_{N}(1)$ of each other, so the result follows from the inverse function theorem. (Notice that the factor $(1-\varrho_2)$ changed to $(1-2\varrho_2)$ above.) 
\end{proof}

\begin{lemma}
\label{lem:cauchy}
For $1\leq m \leq L$ and any good sequence $(q_L,\dots,q_0)$, we have
\[
    \sum_{\ell=0}^{m-1}
    (q_{\ell}-q_{\ell+1})^2 
    \geq 
    \frac{m \cdot d_F(\alpha,\beta)^2 \varrho^2(1-4\varrho_2)\cdot q_0(1-q_0)}
    {\log^2 N}.
\]
\end{lemma}

\begin{proof}

The result follows from Lemma~\ref{lem:ql-diff} and Cauchy-Schwarz in the form
\[
    \sum_{\ell=0}^{m-1}
    (q_{\ell}-q_{\ell+1})^2 
    \geq 
    m^{-1}
    \left(\sum_{\ell=0}^{m-1}
    |q_{\ell}-q_{\ell+1}|\right)^2.
\]
\end{proof}

\begin{lemma}
For any good sequence $(q_L,\dots,q_0)$ and $j\geq 0$, we have
\[
    \sum_{\ell=0}^{L-1} 
    b_{k_0+j-\ell}(q_{\ell}-q_{\ell+1})^2
    \geq
    (1-O(\varrho_2))\cdot 
    \frac{b_{k_0+j}\varrho \, d_F(\alpha,\beta)^2\cdot q_0(1-q_0) }
    {\log^2 N}
    .
\]
\end{lemma}

\begin{proof}

We break the sum into parts and apply Lemma~\ref{lem:cauchy} to each one. We have:
\begin{align*}
   \sum_{\ell=0}^{L-1} 
    b_{k_0+j-\ell}(q_{\ell}-q_{\ell+1})^2
    &=
    b_{k_0+j-L+1}
    \sum_{\ell=0}^{L-1} 
    (q_{\ell}-q_{\ell+1})^2
    +
    \sum_{m=1}^{L-1}
    (b_{k_0+j-m+1}-b_{k_0+j-m})
    \sum_{\ell=0}^{m-1}
    (q_{\ell}-q_{\ell+1})^2
    \\
    &\geq
    \sum_{m=1}^{L-1}
    b_{k_0+j}
    \cdot 
    \frac{\varrho}{(1+\varrho)^{m+10}}
    \cdot
    (1-4\varrho_2)
    \frac{m \varrho^2 d_F(\alpha,\beta)^2\cdot q_0(1-q_0) }
    {\log^2 N}
    \\
    &\geq
    (1-O(\varrho+\varrho_2))\cdot
    b_{k_0+j}
    \cdot
    \frac{\varrho^3 d_F(\alpha,\beta)^2 \cdot q_0(1-q_0) }
    {\log^2 N}
    \cdot
    \sum_{m=1}^{L-1}
    \frac{m}{(1+\varrho)^{m}}.
\end{align*}
For $L=L(\varrho)=O\big(\varrho^{-1}\log(\varrho^{-1})\big)$ sufficiently large,
\begin{align*}
    \sum_{m=1}^{L-1}
    \frac{m\varrho}{(1+\varrho)^{m}}
    &\geq
    (1-\varrho)
    \sum_{m=1}^{\infty}
    \frac{m}{(1+\varrho)^{m}}.
    \\
    &=
    (1-\varrho)\left(\sum_{m=1}^{\infty}
    \frac{1}{(1+\varrho)^{m}}\right)^2
    \\
    &=
    \frac{1-\varrho}{\varrho^2}
    .
\end{align*}
Substituting and recalling that $\varrho\ll\varrho_2$ completes the proof. 
\end{proof}

Combining with Lemma~\ref{lem:general-good-seq-bound} yields the second inequality below (the first is trivial).

\begin{corollary}
\label{cor:uniform-good-seq-bound}
For any $\mu$ and $q_0$, we have
\begin{align*}
    &\bbP^{p_i\sim\mu}
    \Big[ 
   \big(\hat p_{i,b_{k_0+j-L}},~
    \hat p_{i,b_{k_0+j-L+1}},~\dots,~
    \hat p_{i,b_{k_0+j}}\big)=\big(q_L,q_{L-1},\dots,q_0\big)
    \Big]
    \\
   &\leq 
   \bbP\Big[ 
   \big(\hat p_{i,b_{k_0+j-L}},~
    \hat p_{i,b_{k_0+j-L+1}},~\dots,~
    \hat p_{i,b_{k_0+j}}\big)=\big(q_L,q_{L-1},\dots,q_0\big)
    ~\big|~p_{i,b_{k_0+j}}=q_0
    \Big]
    \\
    &\leq
     \exp\Big( -\big(1-O(\varrho_2)\big)
    \frac{b_{k_0+j} d_F(\alpha,\beta)^2}
    {2\log^2 N}
    \Big).
\end{align*}
\end{corollary}

\begin{lemma}
\label{lem:main-tail-bound}
Let $j_0$ be the largest $j$ such that $b_{k_0+j}\leq N$. Then for $N$ sufficiently large,
\[
    \sum_{j=1}^{j_0}
    \mathbb E\left[
    e^{X_i\cdot \frac{c_{\alpha,\beta}-\varrho_3}{\log^2 N}}
    \cdot
    I_i^{b_{k_0+j}}
    \right]
    \leq
    c(\alpha,\varrho)/4.
\]
\end{lemma}

\begin{proof}

Recall that $c_{\alpha,\beta}=\frac{d_F(\alpha,\beta)^2}{2}$, and observe that the number of total sequences $(q_L,\dots,q_0)\in [0,1]^{L+1}$ with $b_{k_0+j+\ell}q_{\ell}\in \bbZ$ is at most $N^{L+1}$ for each $j\leq j_0$. 
Combining Lemma~\ref{lem:bad-seq-bound} and Corollary~\ref{cor:uniform-good-seq-bound} and noting that the latter always gives the main contribution, we find for each $j\leq j_0$,
\begin{align*}
   \mathbb E\left[
    e^{X_i\cdot \frac{c_{\alpha,\beta}-\varrho_3}{\log^2 N}}
    \cdot
    I_i^{b_{k_0+j}}
    \right]
    &\leq
    N^{L+1}
    \exp\lt( 
    \frac{b_{k_0+j}}{\log^2 N}\cdot 
    \big((c_{\alpha,\beta}-\varrho_3) -(1-O(\varrho_2))c_{\alpha,\beta}\big)
    \rt)
    \\
    &\leq
    \exp\lt(
    -\Omega\left(\frac{\varrho_3 b_{k_0+j}}{\log^2 N}\right)
    \rt)
\end{align*}
so long as $\varrho_3$ is chosen so that $\varrho_3\gg \max(\varrho,\varrho_2)$. In the last line we used the fact that $b_{k_0+j}\geq b_{k_0}\geq \log^4 N$ to absorb the factor $N^{L+1}\leq e^{\varrho \log^{3/2} N}$ for large $N$. Summing over $j$ gives the desired result, since for $\varrho_4=\Omega(\varrho_3)$ and $N$ sufficiently large,
\begin{align*}
     \sum_{j=1}^{\infty} 
     e^{
    -\Omega\left(\frac{\varrho_3 b_{k_0+j}}{\log^2 N}\right)
    }
    &\leq
    \sum_{m=1}^{\infty}
    e^{
    -\frac{\varrho_4 (m+b_{k_0})}{\log^2 N}
    }
    \\
    &=
    e^{-\varrho_4 \log^2 N}
     \sum_{m=1}^{\infty}
     e^{-\frac{\varrho_4 m}{\log^2 N}}
     \\
     &\leq
     e^{-\varrho_4 \log^2 N}
     \cdot
     O\left(\frac{\log^2 N}{\varrho_4}\right)
     \\
     &\leq
     e^{-\frac{\varrho_4 \log^2 N}{2}}
     \\
     &\leq 
     c(\alpha,\varrho)/4.     
\end{align*}
\end{proof}

We now use Lemma~\ref{lem:doob-maximal} to conclude.

\begin{proof}[Proof that Algorithm~\ref{alg:main} achieves the guarantee of Theorem~\ref{thm:main}]

By combining Lemma~\ref{lem:main-tail-bound} with the previous Propositions~ \ref{prop:t0-bound} and \ref{prop:k0-tail-bound}, it follows that 
\[
    \mathbb E\left[
    e^{X_i\cdot \frac{c_{\alpha,\beta}-\varrho_3}{\log^2 N}}
    \cdot
    I_i
    \right]
    \leq 1.
\]

Lemma~\ref{lem:doob-maximal} now implies that the total amount of time spent on eventually rejected arms is at most $N(1-\varrho)$ with probability
\[
    e^{-\frac{(c_{\alpha,\beta}-\varrho_3)(1-\varrho)N}{\log^2 N}}.
\]
On this event, the output arm $i^*$ satisfies $n_{i^*,N}\geq \varrho N$ by definition. Since $i^*$ was not rejected, for $j_1$ be the largest value such $b_{k_0+j_1}\leq \varrho N$ we have
\[
    \hat p_{i^*,b_{k_0+j_1}}\geq \beta+\varrho.
\]
The probability for this to hold if $p_i\leq \beta$ is at most $e^{-\Omega_{\varrho}(N)}$. Altogether we find that
\begin{equation}
\label{eq:final}
    \bbP[p_{i^*}\geq \beta]\geq 1- \exp\lt(-\frac{(c_{\alpha,\beta}-\varrho_5)N}{\log^2 N}\rt)-e^{-\Omega_{\varrho}(N)}
\end{equation}
for $\varrho_5$ arbitrarily small. This concludes the analysis of Algorithm~\ref{alg:main} (since the last error term is negligible).
\end{proof}

\subsection{Finding Many Good Arms with a Fixed Budget}
\label{subsec:many-good-fixed-budget}

In this final subsection we observe that Algorithm~\ref{alg:main} can be modified to output as many as $\log N$ distinct arms each of which satisfies the same $(\eta,\eps,\delta)$-PAC guarantee\footnote{In fact $\log N$ can be replaced by anything $o_N(\log^2 N)$ by more precisely defining $M$ and $\tilde N$.}, with no degradation in the asymptotic failure probability. With other parameters fixed, we denote the $N$-sample version of Algorithm~\ref{alg:main} by $\cA_N$ to emphasize the dependence on $N$. In particular, $N$ both equals the number of steps in $\cA_N$ and appears (via its logarithm) in the description of $\cA_N$'s individual steps.

Let $\tilde N=N+\lceil\frac{2N}{\log^{1/2}(N)}\rceil$. We consider a modified algorithm $\tilde\cA_{\tilde N}$ which mimicks the behavior of $\cA_N$ with two changes:
\begin{enumerate}
    \item $\tilde\cA_{\tilde N}$ is a $\tilde N$-sample algorithm.
    \item If an arm $a_i$ has not yet been rejected after $M=\lceil N/\log^{3/2}(N)\rceil$ samples, then $\tilde\cA_{\tilde N}$ accepts $a_i$ and continues to $a_{i+1}$. In particular, $\tilde\cA_{\tilde N}$ may accept several arms instead of just one.
\end{enumerate}

\begin{theorem}
\label{thm:return-many-arms}
    With probability $1- \exp\lt(-\frac{(c_{\alpha,\beta}-\varrho_5-o_N(1))N}{\log^2 N}\rt)$,  $\tilde\cA_{\tilde N}$ accepts at least $\log(N)$ distinct arms $a_i$, all of which satisfy $p_i\geq \beta$.
\end{theorem}

The change from $N$ to $\tilde N$ is almost irrelevant in the actual statement of Theorem~\ref{thm:return-many-arms} since $\log(N)\geq \log(\tilde N)-o_N(1)$. In particular, $\tilde\cA_{\tilde N}$ is a $\tilde N$-sample algorithm which outputs at least $\log(\tilde N)-1$ arms with probability $1- \exp\lt(-\frac{(c_{\alpha,\beta}-\varrho_5-o_{\tilde N}(1))\tilde N}{\log^2 \tilde N}\rt)$. It is certainly not really necessary to use the value $\log(N)$ rather than $\log(\tilde N)$ to describe the individual steps taken by $\tilde A_{\tilde N}$. However introducing $\tilde N$ streamlines the proof below by letting us treat $\cA_N$ as a blackbox.

\begin{proof}
    To show that all accepted arms $a_i$ satisfy $p_i\geq\beta$ with sufficiently high probability, it suffices to consider \eqref{eq:final} with the final term replaced by $e^{-\Omega_{\varrho}(N/\log^{3/2}(N))}$. In particular, observe that the main term does not change, even after multiplying the failure probability by $O\big(\log^{3/2}(N)\big)$ (the maximum possible number of arms accepted by $\tilde\cA_{\tilde N}$. Thus we focus on showing that $\tilde\cA_{\tilde N}$ outputs at least $\log(N)$ arms with high probability.

    Consider yet another $N$-sample algorithm $\widehat\cA_N$ which deletes each arm independently with probability $1/N$ and follows $\cA_N$ on the set of non-deleted arms in order of increasing index. (Like $\cA_N$, $\widehat\cA_N$ never accepts arms before time $N$.) We simulate $\tilde \cA_{\tilde N}$ and $\widehat\cA_N$ on the same reward sequences, i.e. we couple them so that the $t$-th sample of arm $a_i$ always gives the same result for each $(t,i)$. We \textbf{claim} that in this coupling, conditioned on $\tilde \cA_{\tilde N}$ failing to accept $\log(N)$ arms within the first $\tilde N$ samples, $\widehat\cA_N$ has probability $\Omega(N^{-\log(N)})$ to fail (i.e. output $a_i$ with $p_i<\beta$) when run for $N$ samples.

    First let us assume the claim and deduce Theorem~\ref{thm:return-many-arms}. Denote by $p(N)$ the probability for $\cA_N$ to fail. Note that $\widehat\cA_N$ has the same failure probability $p(N)$, having in fact the same behavior as $\cA_N$ in distribution (as the set of deleted arms is independent of everything else). Moreover let $\tilde p(\tilde N,k)$ denote the probability that $\tilde \cA_{\tilde N}$ fails to accept at least $k$ arms. The claim above implies 
    \begin{align*}
    \tilde p(\tilde N,\log N)
    &\leq 
    O\big(N^{\log N}\big)\cdot p(N,1) 
    \\
    &\leq
    e^{o_N(N/\log^2 N)}\cdot p(N,1)
    \\
    &\leq 
    \exp\lt(-\frac{(c_{\alpha,\beta}-\varrho_5-o_N(1))N}{\log^2 N}\rt).
    \end{align*}

    It remains to prove the above claim. Let us say the infinite \iid reward sequence $(r_{i,n})_{n\geq 1}$ of arm $a_i$ is \textbf{acceptable} if $\cA_N$ would not reject $a_i$ within $M$ samples, i.e. $\tilde\cA_{\tilde N}$ will either accept $a_i$ or run out of samples before doing so. We take the point of view that each $a_i$ is either acceptable or not (by randomly fixing the reward sequences at the start). Then with probability $\Omega(N^{-\log(N)})$, the first $\log(N)$ acceptable arms are skipped by $\widehat\cA$, and the first $\hat N$ unacceptable arms are not skipped. On this event, the first $\hat N-M\geq N$ samples obtained by $\widehat\cA_N$, i.e. all $N$ of its samples, are drawn from unacceptable arms. On this event, $\widehat\cA_N$ fails with constant probability, which establishes the claim and completes the proof.
\end{proof}

\section*{Acknowledgements}

X.G. was supported by an Accenture Fellowship.
M.S. thanks Ofer Grossman for suggesting the fixed budget problem and for helpful discussions, as well as Dana Moshkovitz for early comments.


\bibliographystyle{alpha} 
\bibliography{bib}

\end{document}